\documentclass[10pt,twocolumn,letterpaper]{article}

\usepackage[pagenumbers]{cvpr} 

\usepackage{graphicx}
\usepackage{amsmath}
\usepackage{amssymb}
\usepackage{booktabs}

\usepackage{bm}

\usepackage[table,dvipsnames]{xcolor}
\usepackage{float}
\usepackage{nicefrac}
\usepackage[percent]{overpic}
\usepackage{xcolor,soul}
\sethlcolor{white}
\usepackage{paralist} 
\usepackage{amsthm}

\usepackage[pagebackref,breaklinks]{hyperref}

\newtheorem{claim}{Claim}

\usepackage[capitalize]{cleveref}
\crefname{section}{Sec.}{Secs.}
\Crefname{section}{Section}{Sections}
\Crefname{table}{Table}{Tables}
\crefname{table}{Tab.}{Tabs.}

\def\cvprPaperID{516} 
\def\confName{CVPR}
\def\confYear{2022}

\newcommand{\uvomega}{\hat{\boldsymbol{\omega}}}
\newcommand{\uvn}{\hat{\mathbf{n}}}
\newcommand{\vN}{\mathbf{N}}
\newcommand{\uvN}{\hat{\mathbf{N}}}
\newcommand{\uvc}{\hat{\mathbf{c}}}
\newcommand{\rgb}{\mathbf{c}}
\newcommand{\uvd}{\hat{\mathbf{d}}}
\newcommand{\vo}{\mathbf{o}}
\newcommand{\vx}{\mathbf{x}}
\newcommand{\vb}{\mathbf{b}}
\newcommand{\refdir}{\hat{\boldsymbol{\omega}}_r}  
\newcommand{\viewdir}{\hat{\mathbf{d}}}  
\newcommand{\outdir}{\hat{\boldsymbol{\omega}}_o}  
\newcommand{\lightdir}{\hat{\boldsymbol{\omega}}_i}  
\newcommand{\density}{\tau}  
\newcommand{\tint}{\mathbf{s}}  

\definecolor{yellow}{rgb}{1,1, 0.7}
\definecolor{lightyellow}{rgb}{1,1, 0.8}
\definecolor{orange}{rgb}{1, 0.85, 0.7}
\definecolor{tablered}{rgb}{1, 0.7, 0.7}

\hypersetup{
    colorlinks=true,
    linkcolor=black,
    citecolor=black,
    filecolor=black,
    urlcolor=black,
}

\begin{document}

\title{Ref-NeRF: Structured View-Dependent Appearance for Neural Radiance Fields}

\author{
Dor Verbin$^{1, 2}$
\quad
Peter Hedman$^2$
\quad
Ben Mildenhall$^2$
\\
Todd Zickler$^1$
\quad
Jonathan T. Barron$^2$
\quad
Pratul P. Srinivasan$^2$
\\
\vspace{2mm}
{$^1$Harvard University \quad $^2$Google Research}
}

\maketitle

\begin{abstract}

Neural Radiance Fields (NeRF) is a popular view synthesis technique that represents a scene as a continuous volumetric function, parameterized by multilayer perceptrons that provide the volume density and view-dependent emitted radiance at each location. While NeRF-based techniques excel at representing fine geometric structures with smoothly varying view-dependent appearance, they often fail to accurately capture and reproduce the appearance of glossy surfaces. We address this limitation by introducing Ref-NeRF, which replaces NeRF's parameterization of view-dependent outgoing radiance with a representation of reflected radiance and structures this function using a collection of spatially-varying scene properties. We show that together with a regularizer on normal vectors, our model significantly improves the realism and accuracy of specular reflections. Furthermore, we show that our model's internal representation of outgoing radiance is interpretable and useful for scene editing.

\end{abstract}

\begin{figure}[!ht]
    \centering
    \begin{tabular}[width=\linewidth]{ccc}
    \quad \scriptsize Ground Truth & \qquad\qquad \scriptsize Ours  & \qquad\qquad \scriptsize Mip-NeRF~\cite{barron2021mipnerf}  \end{tabular}
    \includegraphics[width=\linewidth]{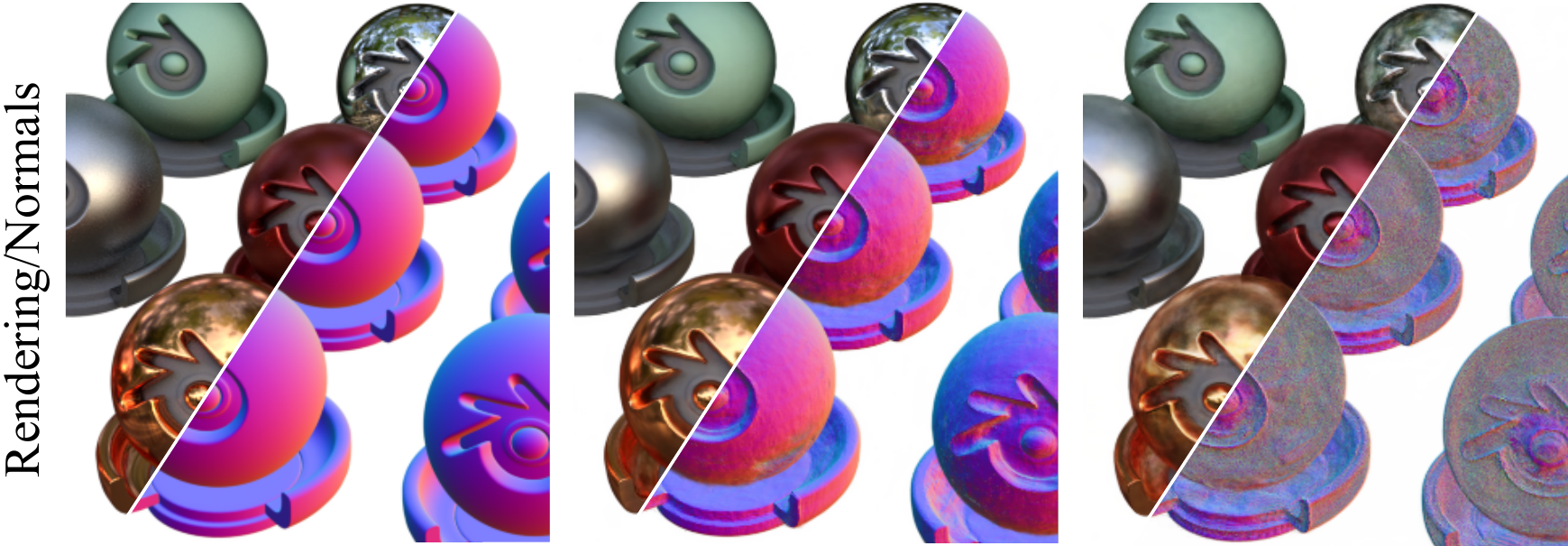} \\
    \begin{tabular}[width=\linewidth]{ccc}
    \qquad \raisebox{5px}{\scriptsize PSNR$\uparrow$/MAE$\downarrow$:} \qquad & \quad\, \raisebox{5px}{\scriptsize $\mathbf{35.6}$\textbf{dB}/$\mathbf{11.5^\circ}$} & \qquad \raisebox{5px}{\scriptsize $30.3$dB/$59.5^\circ$} \qquad\qquad\quad
    \end{tabular}
    \includegraphics[width=\linewidth]{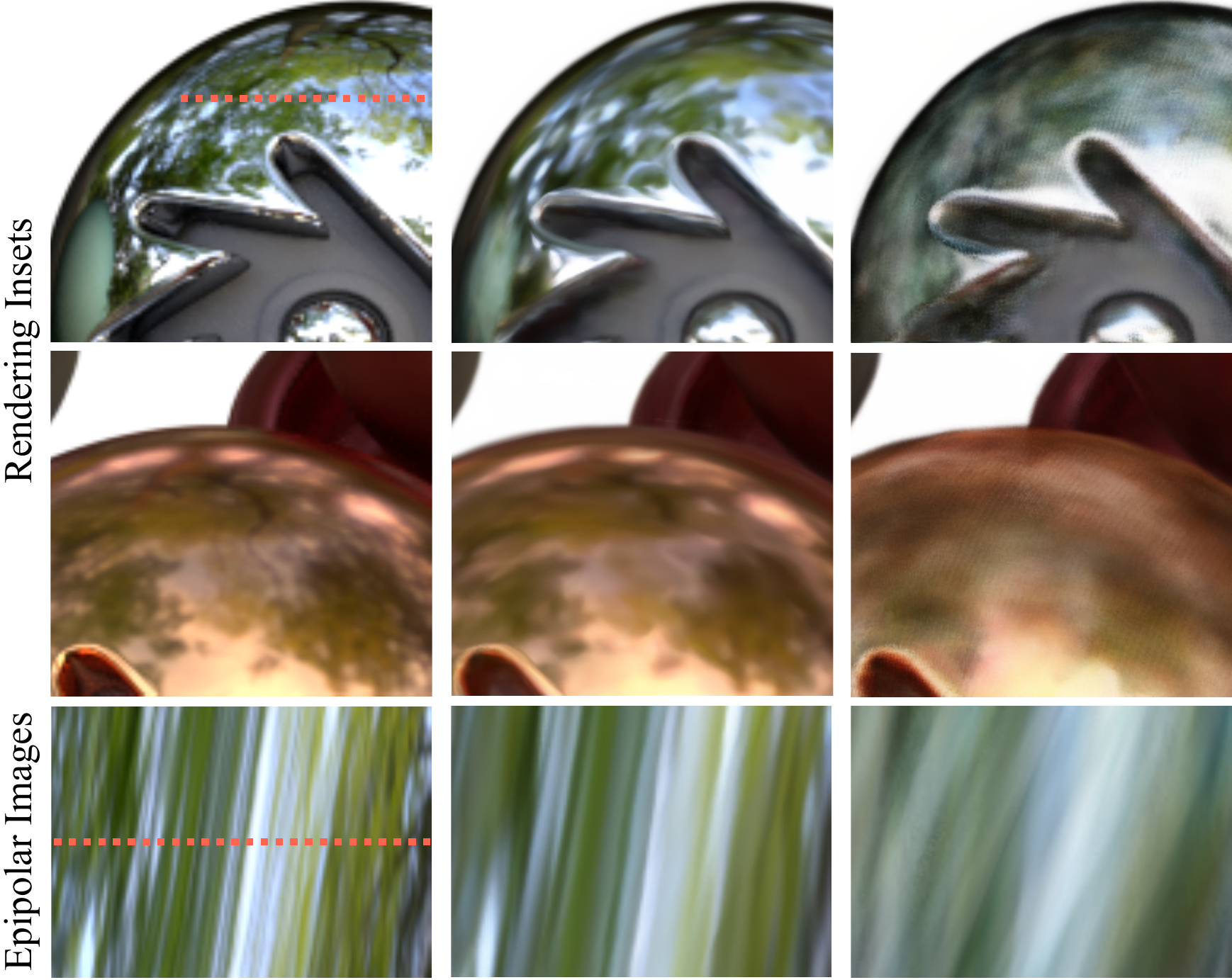}
    \caption{
    Ref-NeRF significantly improves normal vectors (top row) and visual realism (remaining rows) compared to mip-NeRF, the previous top-performing neural view synthesis model. Ref-NeRF's improvements are apparent in rendered frames (Rows 2 \& 3), and even more in rendered videos (bottom row epipolar plane images and supplementary video), where its glossy highlights shift realistically across views instead of blurring and fading like mip-NeRF's.
    Image PSNR (higher is better) and surface normal mean angular error (lower is better) shown as insets.
    }
    \label{fig:teaser}
\end{figure}

\section{Introduction}
\label{sec:intro}

Neural Radiance Fields (NeRF)~\cite{mildenhall2020nerf} renders compelling photorealistic images of 3D scenes from novel viewpoints using a neural volumetric scene representation. Given any input 3D coordinate in the scene, a ``spatial'' multilayer perceptron (MLP) outputs the corresponding volume density at that point, and a ``directional'' MLP outputs the outgoing radiance at that point along any input viewing direction. Although NeRF's renderings of view-dependent appearance may appear reasonable at first glance, a close inspection of specular highlights reveals spurious glossy artifacts that fade in and out between rendered views (Figure~\ref{fig:teaser}), rather than smoothly moving across surfaces in a physically-plausible manner.

These artifacts are caused by two fundamental issues with NeRF (and top-performing extensions such as mip-NeRF~\cite{barron2021mipnerf}). First, NeRF's parameterization of the outgoing radiance at each point as a function of the viewing direction is poorly-suited for interpolation. Figure~\ref{fig:reparam} illustrates that, even for a simple toy setup, the scene's true radiance function varies quickly with view direction, especially around specular highlights. As a consequence, NeRF is only able to accurately render the appearance of scene points from the specific viewing directions observed in the training images, and its interpolation of glossy appearance from novel viewpoints is poor. Second, NeRF tends to ``fake'' specular reflections using isotropic emitters inside the object instead of view-dependent radiance emitted by points at the surface, resulting in objects with semitransparent or ``foggy'' shells.

Our key insight is that structuring NeRF's representation of view-dependent appearance can make the underlying function simpler and easier to interpolate. We present a model, which we call Ref-NeRF, that reparameterizes NeRF's directional MLP by providing the reflection of the viewing vector about the local normal vector as input instead of the viewing vector itself. Figure~\ref{fig:reparam} (left column) illustrates that, for a toy scene comprised of a glossy object under distant illumination, this \emph{reflected radiance} function is constant across the scene (ignoring lighting occlusions and interreflections) because it is unaffected by changes in surface orientation. Consequently, since the directional MLP acts as an interpolation kernel, our model is better able to ``share'' observations of appearance between nearby points to render more realistic view-dependent effects in interpolated views. We additionally introduce an \emph{Integrated Directional Encoding} technique, and we structure outgoing radiance into explicit diffuse and specular components to allow the reflected radiance function to remain smooth despite variation in material and texture over the scene.

While these improvements crucially enable Ref-NeRF to accurately interpolate view-dependent appearance, they rely on the ability to reflect viewing vectors about normal vectors estimated from NeRF's volumetric geometry. This presents a problem, as NeRF's geometry is foggy and not tightly concentrated at surfaces, and its normal vectors are too noisy to be useful for computing reflection directions (as shown in the right column of Figure~\ref{fig:teaser}). We ameliorate this issue with a novel regularizer for volume density that significantly improves the quality of NeRF's normal vectors and encourages volume density to concentrate around surfaces, enabling our model to compute accurate reflection vectors and render realistic specular reflections, shown in Figure~\ref{fig:teaser}.

\begin{figure}
    \centering
    \includegraphics[width=\linewidth]{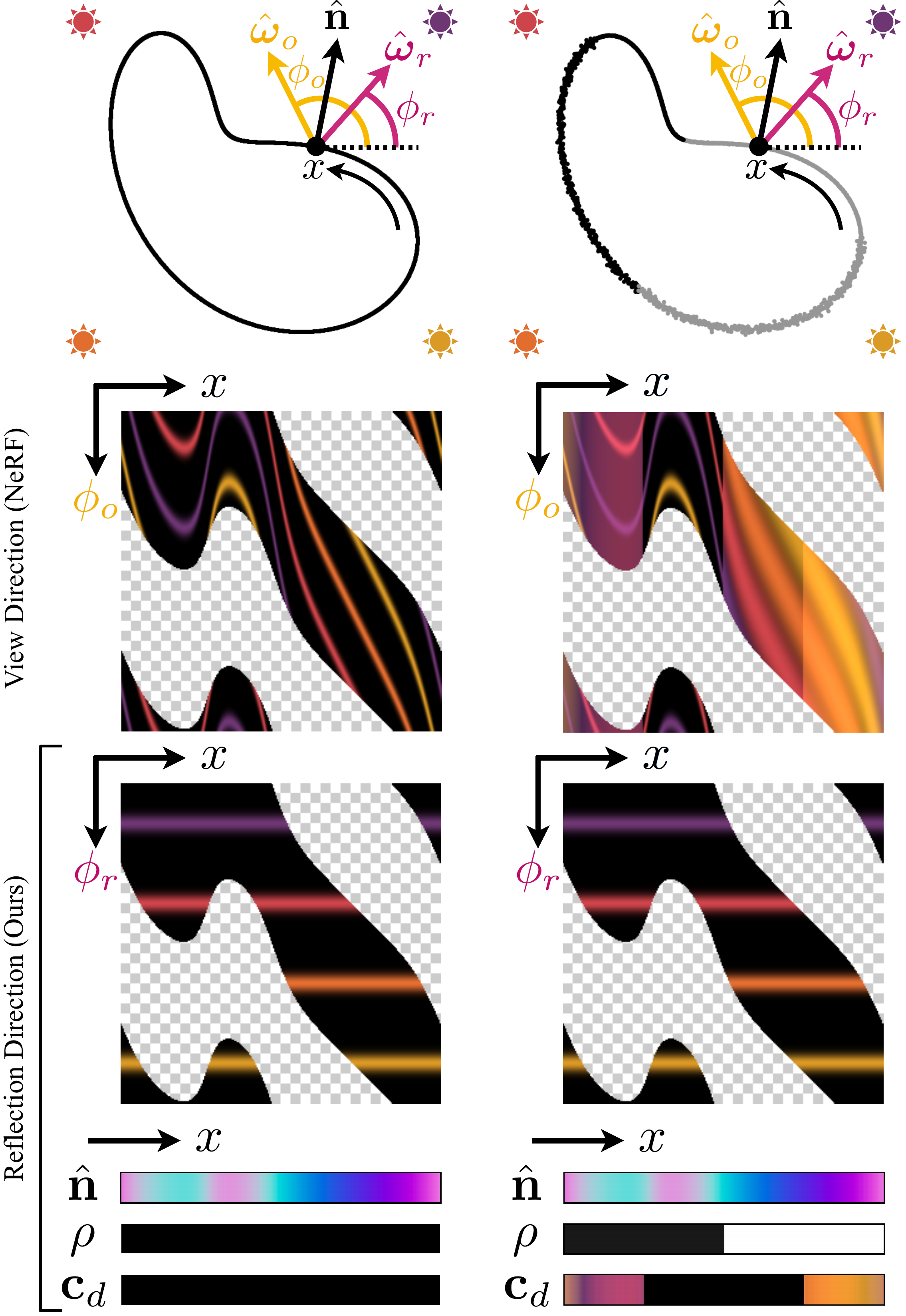}
    \caption{Visualizations of outgoing radiance in NeRF and Ref-NeRF, using 2D position-angle slices of radiance along an $x$-parameterized surface curve on a glossy object under colored lights.
    Because NeRF (middle row)  uses view angle $\phi_o$ as input, when presented with glossy reflectances (left) or spatially-varying materials (right) it must interpolate between highly complicated functions like the irregularly curved colored lines shown here.
    In contrast, Ref-NeRF (bottom row) parameterizes radiance using a normal vector $\uvn$ and a reflection angle $\phi_r$, and adds diffuse color $\rgb_{d}$ and roughness $\rho$ to its spatial MLP, which collectively makes radiance functions simple to model even for shiny or spatially-varying materials.
    The gray checkerboard indicates directions below the surface at position $x$. 
    } 
    \label{fig:reparam}
\end{figure}

To summarize, we make the following contributions:

\begin{compactenum}
    \item A reparameterization of NeRF's outgoing radiance, based on the reflection of the viewing vector about the local normal vector (Section~\ref{sec:refl}).
    \item An Integrated Directional Encoding (Section~\ref{sec:ide}) that, when coupled with a separation of diffuse and specular colors (Section~\ref{sec:diffspec}), enables the reflected radiance function to be smoothly interpolated across scenes with varying materials and textures.
    \item A regularization that concentrates volume density around surfaces and improves the orientation of NeRF's normal vectors (Section~\ref{sec:normals}).
\end{compactenum}

We apply these changes on top of mip-NeRF~\cite{barron2021mipnerf}, currently the top-performing neural representation for view synthesis. Our experiments demonstrate that Ref-NeRF produces state-of-the-art renderings of novel viewpoints, and substantially improves upon the quality of previous top-performing view synthesis methods for highly specular or glossy objects. Furthermore, our structuring of outgoing radiance produces interpretable components (normal vectors, material roughness, diffuse texture, and specular tint) that enable convincing scene editing capabilities.

\section{Related Work}
\label{sec:related}

We review NeRF and related methods for photorealistic view synthesis, as well as techniques from computer graphics for capturing and rendering specular appearance.

\paragraph{3D scene representations for view synthesis}

View synthesis, the task of using observed images of a scene to render images from novel unobserved camera viewpoints, is a longstanding research problem within the fields of computer vision and graphics. 
In situations where it is possible to densely capture images of the scene, simple light field interpolation techniques~\cite{gortler1996lumigraph,levoy1996light} can render novel views with high fidelity. However, exhaustive sampling of the light field is impractical in most scenarios, so methods for view synthesis from sparsely-captured images reconstruct 3D scene geometry in order to reproject observed images into novel viewpoints~\cite{chen1993ibr}.
For scenes with glossy surfaces, some methods explicitly build virtual geometry to explain the motion of reflections~\cite{sinha2012,kopf2013,rodriguez2020semantic}.
Early approaches used triangle meshes as the geometry representation, and rendered novel views by reprojecting and blending multiple captured images with either heuristic~\cite{buehler2001unstructured,debevec1996modeling,wood2000surface} or learned~\cite{hedman2018deep,riegler2020free} blending algorithms. Recent works have used volumetric representations such as voxel grids~\cite{lombardi2019neuralvolumes} or multiplane images~\cite{flynn2019deepview,mildenhall2019llff,srinivasan2019pushing,wizadwongsa2021nex,zhou2018stereomag}, which are better suited for gradient-based optimization than meshes. While these discrete volumetric representations can be effective for view synthesis, their cubic scaling limits their ability to represent large or high resolution scenes.

The recent paradigm of \emph{coordinate-based neural representations} replaces traditional discrete representations with an MLP that maps from any continuous input 3D coordinate to the geometry and appearance of the scene at that location. NeRF~\cite{mildenhall2020nerf} is an effective coordinate-based neural representation for photorealistic view synthesis that represents a scene as a field of particles that block and emit view-dependent light. NeRF has inspired many subsequent works, which extend its neural volumetric scene representation to application domains including dynamic and deformable scenes~\cite{park2021nerfies}, avatar animation~\cite{gafni2021dynamic,peng2021animate}, and even phototourism~\cite{martinbrualla2020nerfw}. Our work focuses on improving a core component of NeRF: the representation of view-dependent appearance. We believe that the improvements presented here can be used to improve rendering quality in many of the applications of NeRF described above. 

A key component of our approach considers the reflection of camera rays off NeRF's geometry. This idea is shared by recent works that extend NeRF to enable relighting by decomposing appearance into scene lighting and materials~\cite{bi2020nrf,boss2021nerd,boss2021neuralpil,srinivasan2021nerv,zhang2021physg,zhang2021nerfactor}. Crucially, our model structures the scene into components that are \emph{not} required to have precise physical meanings, and is thus able to avoid the strong simplifying assumptions (such as known lighting~\cite{bi2020nrf,srinivasan2021nerv}, no self-occlusions~\cite{boss2021nerd,boss2021neuralpil,zhang2021physg}, single-material scenes~\cite{zhang2021physg}) that these works need to make to recover explicit parametric representations of lighting and material. Our work also focuses on improving the smoothness and quality of normal vectors extracted from NeRF's geometry. This goal is shared by recent works that combine NeRF's neural volumetric representation with neural implicit surface representations~\cite{oechsle2021unisurf,wang2021neus,yariv2021volume}. However, these methods primarily focus on the quality of isosurfaces extracted from their representation as opposed to the quality of rendered novel views, and as such their view synthesis performance is significantly worse than top-performing NeRF-like models.

\paragraph{Efficient rendering of glossy appearance}

Our work takes inspiration from seminal approaches in computer graphics for representing and rendering view-dependent specular and reflective appearance, particularly techniques based on precomputation~\cite{ramamoorthi2009prt}. The reflected radiance function encoded in our directional MLP is similar to prefiltered environment maps~\cite{kautz2000approximation,ramamoorthi2002frequency}, which were introduced for real-time rendering of specular appearance. Prefiltered environment maps leverage the insight that the outgoing light from a surface can be seen as a spherical convolution of the incoming light and the (radially-symmetric) bidirectional reflectance distribution function (BRDF) that describes the material properties of the surface~\cite{ramamoorthi2001signal}. After storing this convolution result, rays intersecting the object can be rendered efficiently by simply indexing into the prefiltered environment maps with the reflection direction of the viewing vector about the normal vector. 

Instead of rendering predefined 3D assets, our work leverages these computer graphics insights to solve a computer vision problem where we are recovering a renderable model of the scene from images. Furthermore, our directional MLP's representation of reflected radiance improves upon the prefiltered environment map representations used in computer graphics in a critical way: our directional MLP can represent spatial variation in reflected radiance due to spatial variation in both lighting and scene properties such as material roughness and texture, while the techniques described previously require computing and storing discrete prefiltered radiance maps for each possible material.

Our work is also inspired by a long line of works in computer graphics that reparameterize directional functions such as BRDFs~\cite{ramamoorthi2002frequency,rusinkiewiczbrdf} and outgoing radiance~\cite{wood2000surface} for improved interpolation and compression. 

\begin{figure*}[t]
    \centering
    \includegraphics[width=\linewidth]{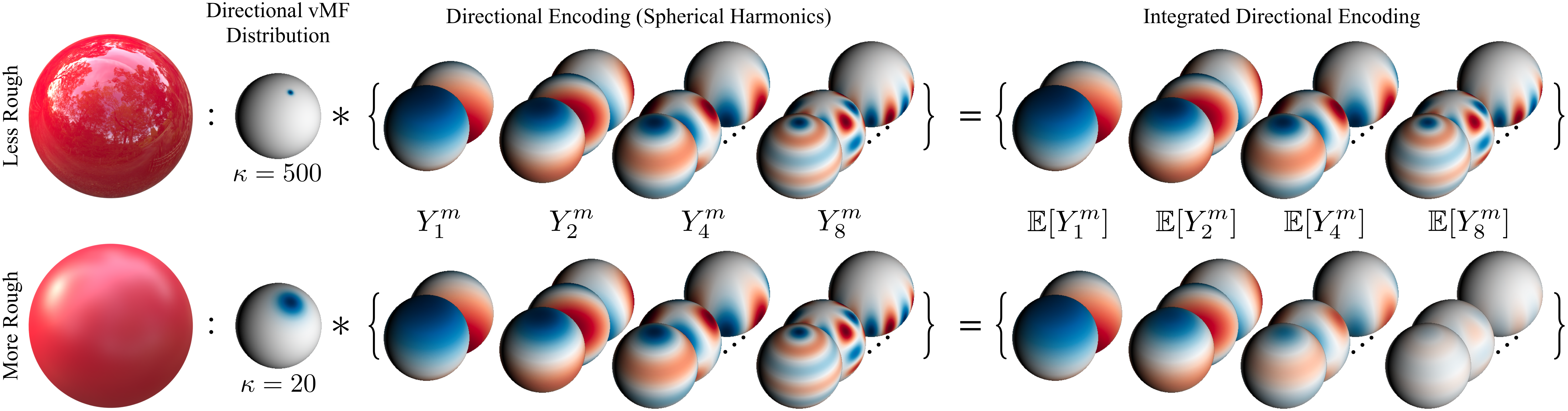}
    \caption{
    We enable the directional MLP to represent reflected radiance functions for any continuously-valued roughness using our integrated directional encoding. Each component of the encoding is a spherical harmonic function convolved with a vMF distribution with concentration parameter $\kappa$, output by our spatial MLP (equivalent to the expectation of the spherical harmonic under the vMF). Less rough locations receive higher-frequency encodings (top), while more rough regions receive encodings with attenuated high frequencies. Our IDE allows lighting information to be shared between locations with different roughnesses, and lets reflectance be edited. 
    }
    \label{fig:ide}
\end{figure*}

\subsection{NeRF Preliminaries}
\label{sec:review}

NeRF~\cite{mildenhall2020nerf} represents a scene as a volumetric field of particles that emit and absorb light. Given any input 3D position $\vx$, NeRF uses a \emph{spatial} MLP to output the density $\density(\vx)$ of volumetric particles as well as a ``bottleneck'' vector $\vb(\vx)$ which, along with the view direction $\viewdir$, is provided to a second \emph{directional} MLP that outputs the color $\rgb(\vx, \viewdir)$ of light emitted by a particle at that 3D position at direction $\viewdir$ (see Figure~\ref{fig:arch} for a visualization). Note that Mildenhall \etal~\cite{mildenhall2020nerf} use a single-layer directional MLP in their work, and that prior work often describes the combination of NeRF's spatial and directional MLPs as a single MLP.

The two MLPs are queried at points $\vx_i = \vo + t_i\viewdir$ along a ray originating at $\vo$ with direction $\viewdir$, and return densities $\{\density_i\}$ and colors $\{\rgb_i\}$. These densities and colors are alpha composited using numerical quadrature~\cite{max1995optical} to obtain the color of the pixel corresponding to the ray:
\begin{align} \label{eq:nerfcolor}
    \mathbf{C}(\vo, \viewdir) &= \sum_{i} w_i \rgb_i \, ,  \\
    \text{where}\quad w_i &= e^{-\sum_{j < i} \density_j (t_{j+1} - t_j)} \left(1 - e^{-\density_i(t_{i+1} - t_i)} \right) \, .\nonumber
\end{align}
The MLP parameters are optimized to minimize the L2 difference between each pixel's predicted color $\mathbf{C}(\vo, \viewdir)$ and its true color $\mathbf{C}_{\mathrm{gt}}(\vo, \viewdir)$ taken from the input image:
\begin{equation} \label{eq:nerfloss}
    \mathcal{L} = \sum_{\vo, \viewdir} \|\mathbf{C}(\vo, \viewdir) - \mathbf{C}_{\mathrm{gt}}(\vo, \viewdir)\|^2 \, .
\end{equation}
In practice, NeRF uses two sets of MLPs, one coarse and one fine, in a hierarchical sampling fashion, where both are trained to minimize the loss in Equation~\ref{eq:nerfloss}. 

Prior NeRF-based models define a normal vector field in the scene by either using the spatial MLP to predict unit vectors~\cite{bi2020nrf,zhang2021nerfactor} at any 3D location, or by using the gradient of the volume density with respect to 3D position~\cite{boss2021nerd,srinivasan2021nerv}:

\begin{equation} \label{eq:normals}
    \uvn(\vx) = -\frac{\nabla \density(\vx)}{\|\nabla \density(\vx)\|} \, .
\end{equation}

\section{Structured View-Dependent Appearance}

In this section, we describe how Ref-NeRF structures the outgoing radiance at each point into (prefiltered) incoming radiance, diffuse color, material roughness, and specular tint, which are better-suited for smooth interpolation across the scene than the function of outgoing radiance parameterized by view direction. By explicitly using these components in our directional MLP (Figure~\ref{fig:arch}), Ref-NeRF can accurately reproduce the appearance of specular highlights and reflections. In addition, our model's decomposition of outgoing radiance enables scene editing.

\subsection{Reflection Direction Parameterization}
\label{sec:refl}

While NeRF directly uses view direction, we instead reparameterize outgoing radiance as a function of the reflection of the view direction about the local normal vector:
\begin{equation} \label{eq:refdir}
    \refdir = 2(\outdir\cdot\uvn)\uvn - \outdir,
\end{equation}
where $\outdir = -\viewdir$ is a unit vector pointing from a point in space to the camera, and $\uvn$ is the normal vector at that point. As demonstrated in Figure~\ref{fig:reparam}, this reparameterization makes specular appearance better-suited for interpolation.

For BRDFs that are rotationally-symmetric about the reflected view direction, \ie ones that satisfy $f(\lightdir,\outdir)=p(\refdir\cdot\lightdir)$ for some lobe function $p$ (which includes BRDFs such as Phong~\cite{phong1975illumination}), and neglecting phenomena such as interreflections and self-occlusions, view-dependent radiance is a function of the reflection direction $\refdir$ only:
\begin{equation} \label{eq:phongrendering}
    L_{\mathrm{out}}(\outdir) \propto \int L_{\mathrm{in}}(\lightdir)p(\refdir\cdot\lightdir) d\lightdir = F(\refdir).
\end{equation}
Thus, by querying the directional MLP with the reflection direction, we are effectively training it to output this integral as a function of $\refdir$. Because more general BRDFs may vary with the angle between the view direction and normal vector due to phenomena such as Fresnel effects~\cite{kautz2000approximation}, we also input $\uvn \cdot \outdir$ to the directional MLP to allow the model to adjust the shape of the underlying BRDF.

\subsection{Integrated Directional Encoding}
\label{sec:ide}

In realistic scenes with spatially-varying materials, radiance cannot be represented as a function of reflection direction alone. The appearance of rougher materials changes slowly with reflection direction, while the appearance of smoother or shinier materials changes rapidly. We introduce a technique, which we call an \emph{Integrated Directional Encoding (IDE)}, that enables the directional MLP to efficiently represent the function of outgoing radiance for materials with any continuously-valued roughness. Our IDE is inspired by the integrated positional encoding introduced by mip-NeRF~\cite{barron2021mipnerf} which enables the spatial MLP to represent prefiltered volume density for anti-aliasing.

First, instead of encoding directions with a set of sinusoids, as done in NeRF, we encode directions with a set of spherical harmonics $\{Y_\ell^m\}$. 
This encoding benefits from being stationary on the sphere,
a property which is crucial to the effectiveness of positional encoding in Euclidean space~\cite{mildenhall2020nerf,tancik2020fourfeat} (more details in our supplement).

Next, we enable the directional MLP to reason about materials with different roughnesses by encoding a distribution of reflection vectors instead of a single vector. We model this distribution defined on the unit sphere with a von Mises-Fisher (vMF) distribution (also known as a normalized spherical Gaussian), centered at reflection vector $\refdir$, and with a concentration parameter $\kappa$ defined as inverse roughness $\kappa=\nicefrac{1}{\rho}$. The roughness $\rho$ is output by the spatial MLP (using a softplus activation) and determines the roughness of the surface: a larger $\rho$ value corresponds to a rougher surface with a wider vMF distribution. Our IDE encodes the distribution of reflection directions using the expected value of a set of spherical harmonics under this vMF distribution:
\begin{align}
    \operatorname{IDE}(\refdir, \kappa) &= \left\{\mathbb{E}_{\uvomega\sim\operatorname{vMF}(\refdir, \kappa)}[Y_\ell^m(\uvomega)] \colon (\ell, m) \in \mathcal{M}_L\right\}, \nonumber\\
    \text{with}\quad \mathcal{M}_L &= \{(\ell, m) \colon \ell=1,...,2^L, m=0,...,\ell\}.
\end{align}

In our supplement, we show that the expected value of any spherical harmonic under a vMF distribution has the following simple closed-form expression:
\begin{equation}
    \mathbb{E}_{\uvomega\sim\operatorname{vMF}(\refdir, \kappa)}[Y_\ell^m(\uvomega)] = A_\ell(\kappa)Y_\ell^m(\refdir),
\end{equation}
and that the $\ell$th attenuation function $A_\ell(\kappa)$ can be well-approximated using a simple exponential function:
\begin{equation} \label{eq:attenuation2}
    A_\ell(\kappa) \approx \exp\left(-\frac{\ell(\ell+1)}{2\kappa}\right).
\end{equation}

Figure~\ref{fig:ide} illustrates that our integrated directional encoding has an intuitive behavior; increasing the roughness of a material by lowering $\kappa$ corresponds to attenuating the encoding's spherical harmonics with high orders $\ell$, resulting in a wider interpolation kernel that limits the high frequencies in the represented view-dependent color.

\begin{figure}[t]
    \centering
    \includegraphics[width=\linewidth]{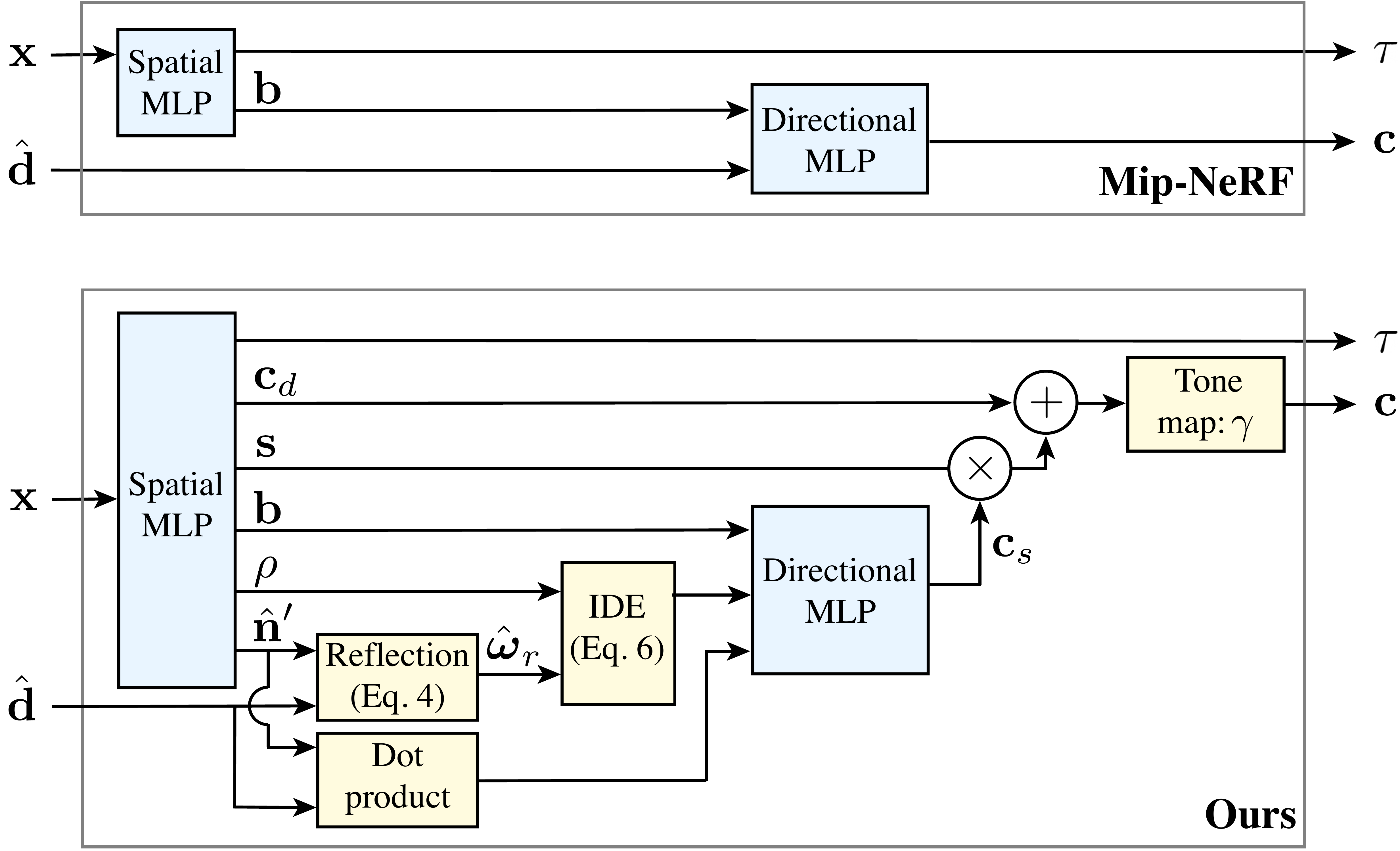}
    \caption{A visualization of mip-NeRF's and our architectures.
    }
    \label{fig:arch}
\end{figure}

\subsection{Diffuse and Specular Colors}
\label{sec:diffspec}

We further simplify the function of outgoing radiance by separating the diffuse and specular components, using the fact that diffuse color is (by definition) a function of only position. We modify the spatial MLP to output a diffuse color $\rgb_{d}$ and a specular tint $\tint$, and we combine this with the specular color $\rgb_{s}$ provided by the directional MLP to obtain a single color value:
\begin{equation}
    \rgb = \gamma(\rgb_{d} + \tint\odot\rgb_{s}),
\end{equation}
where $\odot$ denotes elementwise multiplication, and $\gamma$ is a fixed tone mapping function that converts linear color to sRGB~\cite{anderson1996proposal} and clips the output color to lie in $[0, 1]$.

\subsection{Additional Degrees of Freedom}

Effects such as interreflections and self-occlusion of lighting cause illumination to vary spatially over a scene. We therefore additionally pass a bottleneck vector $\vb$, output by the spatial MLP, into the directional MLP so the reflected radiance can change with 3D position.

\newcommand{\phongid}{199}
\newcommand{\cropamt}{75px}
\begin{figure}[t]
    \centering
    \begin{tabular}{@{}c@{\,}@{}c@{\,}@{\,}c@{\,}@{\,}c@{}}      & \scriptsize Ground truth  & \scriptsize Ours & \scriptsize Mip-NeRF~\cite{barron2021mipnerf}\\
        \rotatebox{90}{\quad \scriptsize Rendering/Normals} &             
        \includegraphics[trim={\cropamt{} \cropamt{} \cropamt{} \cropamt{}},clip,width=0.31\linewidth]{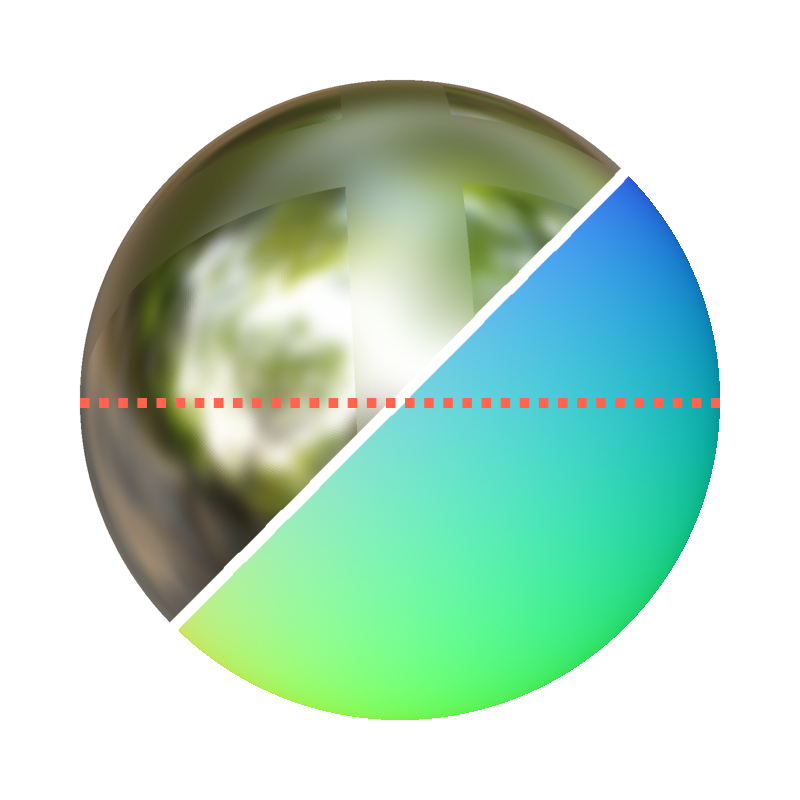} & 
        \includegraphics[trim={\cropamt{} \cropamt{} \cropamt{} \cropamt{}},clip,width=0.31\linewidth]{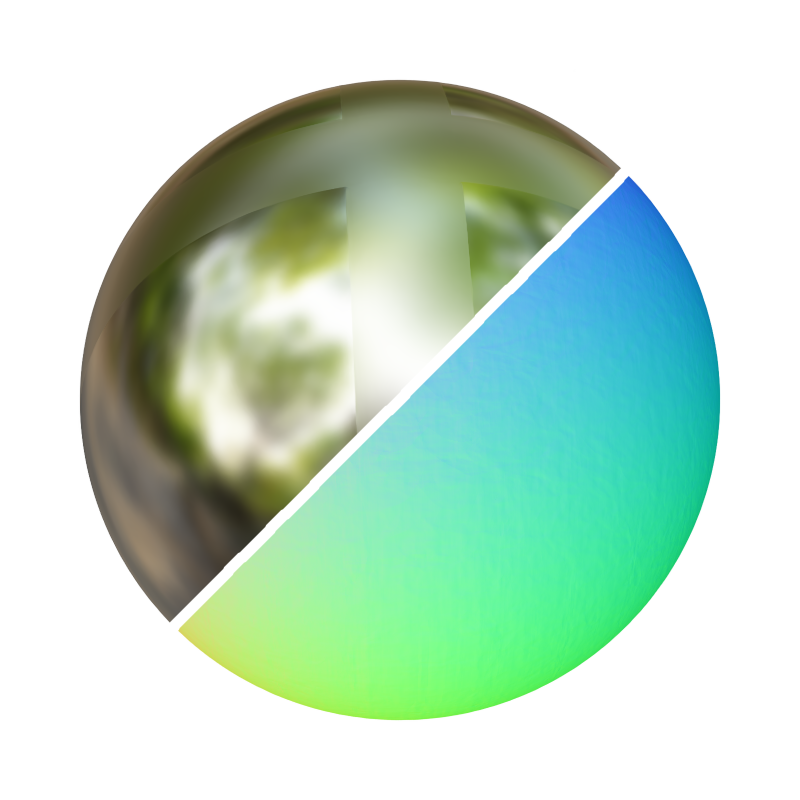} & 
        \includegraphics[trim={\cropamt{} \cropamt{} \cropamt{} \cropamt{}},clip,width=0.31\linewidth]{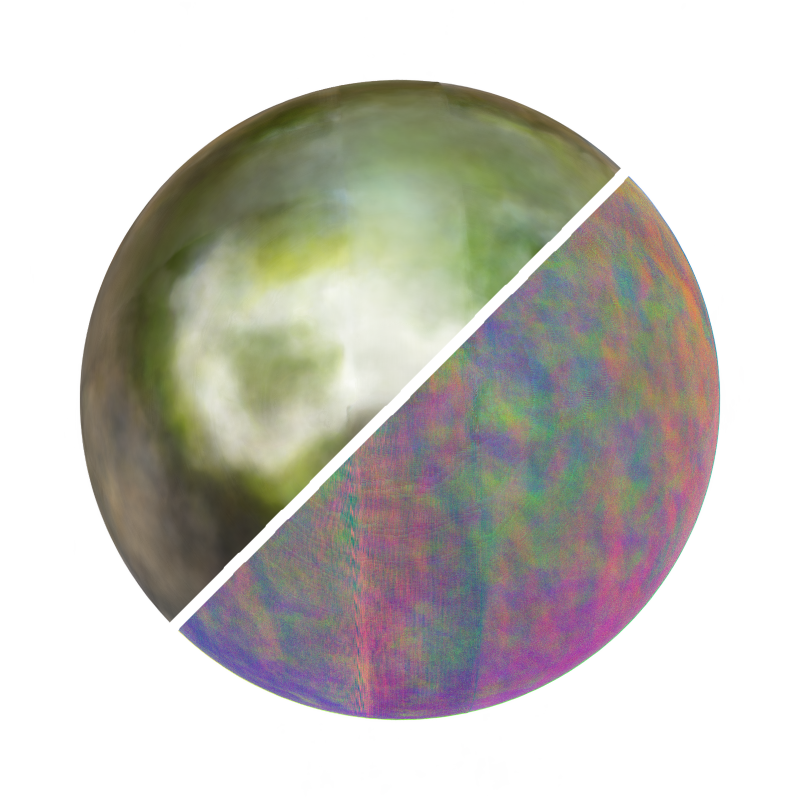} \\
        & \raisebox{5px}{\scriptsize PSNR$\uparrow$/MAE$\downarrow$:} &\raisebox{5px}{\scriptsize $\mathbf{47.7}$\textbf{dB}/$\mathbf{1.6^\circ}$} & \raisebox{5px}{\scriptsize $23.2$dB/$96.6^\circ$} \\
        \rotatebox{90}{\scriptsize Epipolar Plane Image} &
        \includegraphics[width=0.28\linewidth]{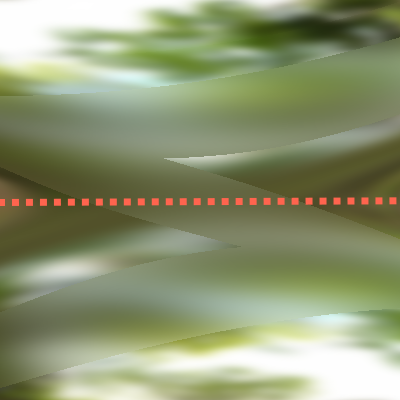} & 
        \includegraphics[width=0.28\linewidth]{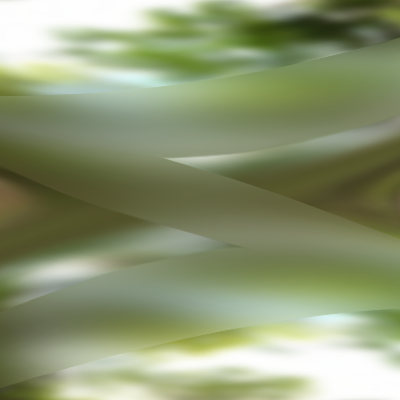} & 
        \includegraphics[width=0.28\linewidth]{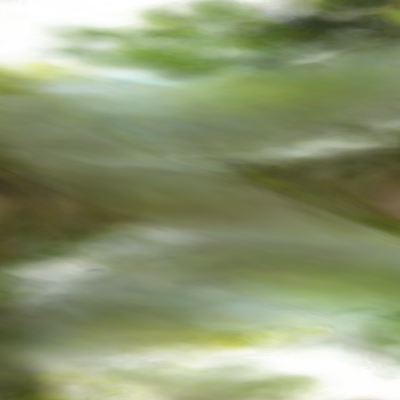} \vspace{0.15cm}
        \\
        \rotatebox{90}{\scriptsize \quad Weights $w_i$} &
        \includegraphics[width=0.29\linewidth]{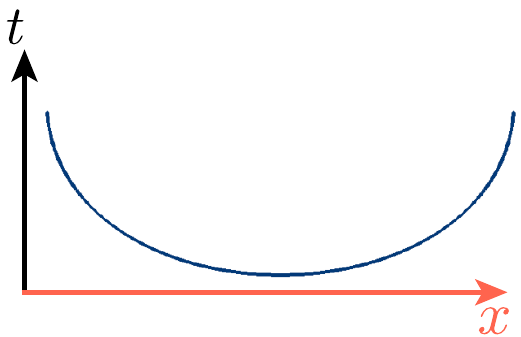} & 
        \includegraphics[width=0.29\linewidth]{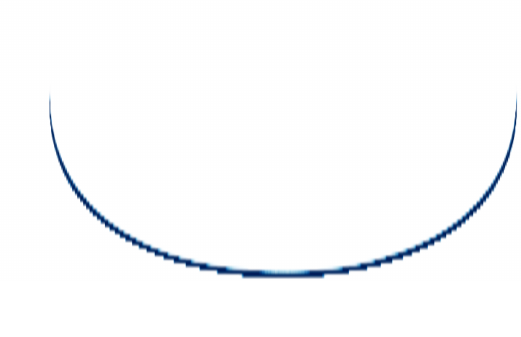} & 
        \includegraphics[width=0.29\linewidth]{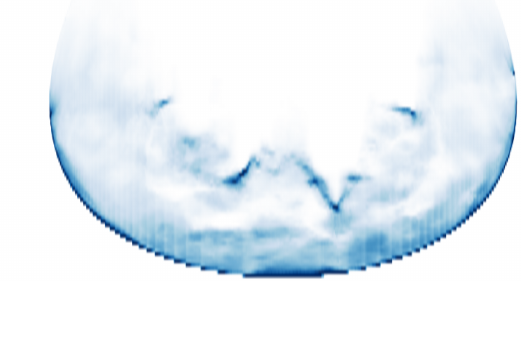}
     \vspace{-0.05cm}
    \end{tabular}
    \hbox{\hspace{1em}\includegraphics[width=0.93\linewidth]{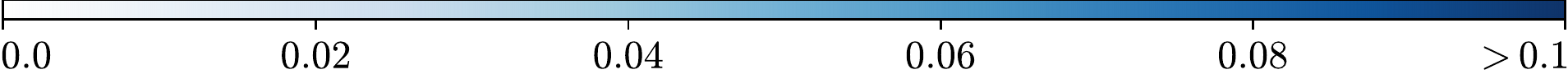}}
    \caption{Prior top-performing NeRF-based approaches can fail catastrophically in highly-reflective scenes. Mip-NeRF (right column) produces blurry renderings of reflections that are inconsistent over different views (see EPI), and does not correctly simulate the appearance of the two different surface roughnesses. Our model (middle column) reconstructs the object almost perfectly. The accumulated normals and rendering weights $w_i$ along the central scanline of the image (bottom row) show that mip-NeRF mimics specularities using emitters inside the object, while Ref-NeRF correctly recovers a concentrated surface.
    }
    \label{fig:phong}
\end{figure}

\section{Accurate Normal Vectors} \label{sec:normals}

While the structuring of outgoing radiance described in the previous section provides a better parameterization for the interpolation of specularities, it relies on a good estimation of volume density for facilitating accurate reflection direction vectors. However, the volume density field recovered by NeRF-based models suffers from two limitations: 1) normal vectors estimated from its volume density gradient as in Equation~\ref{eq:normals} are often extremely noisy (Figures~\ref{fig:teaser} and~\ref{fig:phong}); and 2) NeRF tends to ``fake'' specular highlights by embedding emitters inside the object and partially occluding them with a ``foggy'' diffuse surface (see Figure~\ref{fig:phong}). This is a suboptimal explanation, as it requires diffuse content on the surface to be semitransparent so that the embedded emitter can ``shine through''.

We address the first issue by using predicted normals for computing reflection directions: for each position $\vx_i$ along a ray we output a $3$-vector from the spatial MLP, which we then normalize to get a predicted normal $\uvn_i'$. We tie these predicted normals to the underlying density gradient normal samples along each ray $\{\uvn_i\}$ using a simple penalty:
\begin{equation} \label{eq:predictloss}
    \mathcal{R}_{\mathrm{p}} = \sum_i w_i \|\uvn_i - \uvn'_i\|^2,
\end{equation}
where $w_i$ is the weight of the $i$th sample along the ray, as defined in Equation~\ref{eq:nerfcolor}. These MLP-predicted normals tend to be smoother than gradient density normals because the gradient operator acts as a high-pass filter on the MLP's effective interpolation kernel~\cite{tancik2020fourfeat}. 

We address the second issue by introducing a novel regularization term that penalizes normals that are ``back-facing'', \ie oriented away from the camera, at samples along the ray that contribute to the ray's rendered color:
\begin{equation} \label{eq:orientloss}
    \mathcal{R}_{\mathrm{o}} = \sum_i w_i \max(0, \uvn'_i\cdot\viewdir)^2.
\end{equation}

This regularization acts as a penalty on ``foggy'' surfaces: samples are penalized when they are ``visible'' (high $w_i$) and the volume density is decreasing along the ray (\ie the dot product between $\uvn_i'$ and ray direction $\viewdir$ is positive). This normal orientation penalty prevents our method from explaining specularities as emitters hidden beneath a semitransparent surface, and the resulting improved normals enable Ref-NeRF to compute accurate reflection directions for use in querying the directional MLP.

Throughout the paper, we use the gradient density normals for visualization and quantitative evaluation, as they directly demonstrate the quality of the underlying recovered scene geometry.

\section{Experiments}
\label{sec:results}

We implement our model on top of mip-NeRF~\cite{barron2021mipnerf}, an improved version of NeRF that reduces aliasing. We use the same spatial MLP architecture as mip-NeRF (8 layers, 256 hidden units, ReLU activations), but we use a larger directional MLP (8 layers, 256 hidden units, ReLU activations) than mip-NeRF to better represent high-frequency reflected radiance distributions. Please refer to our supplement for additional baseline implementation details. 

We use the same quantitative metrics as previous view synthesis works~\cite{mildenhall2020nerf,barron2021mipnerf,zhang2021physg}: PSNR, SSIM~\cite{wang2004image}, and LPIPS~\cite{zhang2018lpips} are used for evaluating rendering quality, and mean angular error (MAE) is used for evaluating estimated normal vectors.

\paragraph{Shiny Blender Dataset}
Though the ``Blender'' dataset used by NeRF~\cite{mildenhall2020nerf} contains a variety of objects with complex geometry, it is severely limited in terms of material variety: most scenes are largely Lambertian. To probe more challenging material properties, we have created an additional ``Shiny Blender" dataset with $6$ different glossy objects rendered in Blender under conditions similar to NeRF's dataset (100 training and 200 testing images per scene). The quantitative results in Table~\ref{tab:shiny_mean} highlight the significant advantage of our model over mip-NeRF, the previous top-performing technique, for rendering novel views of these highly specular scenes. We also include three improved versions of mip-NeRF, all of which have an 8-layer directional MLP and, respectively: 1) no additional components; 2) normal vectors appended to the view direction in the directional MLP (as was done in IDR~\cite{yariv2020idr} and VolSDF~\cite{yariv2021volume}); and 3) our orientation loss applied to mip-NeRF's density gradient normal vectors. Our method significantly outperforms all of these improved variants of this previously top-performing neural view synthesis method, both in terms of novel view rendering quality and normal vector accuracy. 
Although PhySG~\cite{zhang2021physg} recovers more accurate normals, it requires ground-truth object masks (all other methods only require RGB images) and produces significantly worse renderings.
Figure~\ref{fig:phong} showcases the impact of our approach using one object from our dataset: while mip-NeRF~\cite{barron2021mipnerf} fails to recover the geometry and appearance of this simple metallic sphere with two roughnesses, our method produces a nearly perfect reconstruction. Figure~\ref{fig:coffee} displays another visual example from this dataset that showcases our model's improvements to recovered normal vectors and rendered specularities. 

\begin{table}[t!]
\centering
\resizebox{\linewidth}{!}{
\begin{tabular}{@{}l|ccc|cr}
& \!PSNR $\uparrow$\! & \!SSIM $\uparrow$\! & \!LPIPS $\downarrow$\! & \!$\text{MAE}^\circ$ $\downarrow$\! \\ \hline
PhySG~\cite{zhang2021physg} (requires object masks)                        &                    26.21 &                    0.921 &                   0.121 & \cellcolor{tablered}    8.46 \\ \hline
Mip-NeRF~\cite{barron2021mipnerf}                  &                    29.76 &                    0.942 &                    0.092 &                    60.38 \\
Mip-NeRF, 8 layers             &                    31.59 &                    0.956 &                    0.072 &                    58.07 \\
Mip-NeRF, 8 layers, w/ normals &                    31.39 &                    0.955 &                    0.074 &                    58.27 \\
Mip-NeRF, 8 layers, w/ $\mathcal{R}_o$ &                    31.48 &                    0.955 &                    0.073 &                    57.37 \\ \hline
Ours, no reflection            &                    29.47 &                    0.944 &                    0.084 & \cellcolor{orange}16.19 \\
Ours, no $\mathcal{R}_o$       &                    31.62 &                    0.954 &                    0.078 &                    52.56 \\
Ours, no pred. normals         &                    30.91 &                    0.936 &                    0.105 &                    30.67 \\
Ours, concat. viewdir          &                    35.42 & \cellcolor{yellow}0.966 &                   0.061 &                    21.25 \\
Ours, fixed lobe               & \cellcolor{yellow}35.52 &                    0.965 &                   0.061 &                    26.46 \\ \hline
Ours, no diffuse color         &                    33.32 &                    0.962 &                    0.067 &                    26.13 \\
Ours, no tint                  &                    35.45 &                    0.965 & \cellcolor{orange}0.060 &                    22.70 \\
Ours, no roughness             &                    33.39 &                    0.963 &                    0.065 &                    25.96 \\
Ours, PE                       & \cellcolor{orange}35.90 & \cellcolor{tablered}0.968 & \cellcolor{tablered}0.058 &                   20.31 \\
Ours                           & \cellcolor{tablered}35.96 & \cellcolor{orange}0.967 & \cellcolor{tablered}0.058 & \cellcolor{yellow}18.38 \\
\end{tabular}

}
\vspace{-0.1in}
\caption{
Baseline comparisons and ablation study on our ``Shiny Blender'' dataset.
}
\label{tab:shiny_mean}
\vspace{-0.2in}
\end{table}

Table~\ref{tab:shiny_mean} also contains a quantitative ablation study of our model. If we use view directions instead of reflection directions as the directional MLP's input (``no reflection"), our method's reconstruction metrics drop significantly, showing the benefits of our reflected radiance parameterization. Removing the orientation loss (``no $\mathcal{R}_o$") results in severely degraded normals and renderings, and applying the orientation loss directly to the density field's normals and using those to compute reflection directions (``no pred. normals'') also reduces performance. Including the view direction as input to the directional MLP in addition to the IDE (``concat. viewdir'') slightly decreases performance, demonstrating the difficulty of parameterizing specular appearance as a function of viewing direction. On the other hand, not feeding $\uvn'\cdot\outdir$ to the directional MLP (``fixed lobe'') also slightly reduces performance. Finally, removing our structural components of outgoing radiance (roughness, diffuse color, or tint) and replacing our IDE with a non-integrated directional encoding or NeRF's standard positional encoding all slightly decrease performance.

\begin{table}[]
\centering
\resizebox{\linewidth}{!}{
\small

\begin{tabular}{@{}l|ccc|cr}
& \!PSNR $\uparrow$\! & \!SSIM $\uparrow$\! & \!LPIPS $\downarrow$\! & \!$\text{MAE}^\circ$ $\downarrow$\! \\ \hline
PhySG~\cite{zhang2021physg} (requires object masks)                        &                    20.60 &                    0.861 &                    0.144  &   29.17 \\
VolSDF~\cite{yariv2021volume}                         &                    27.96 &                    0.932 &                    0.096  & \cellcolor{tablered}19.45 \\
NSVF~\cite{liu2020nsvf}                         &                    31.74 &                    0.953 &                    0.047  &   -- \\
NeRF~\cite{mildenhall2020nerf}                         &                    32.38 &                    0.957 &                    0.046  &   -- \\
Mip-NeRF~\cite{barron2021mipnerf}                       &                    \cellcolor{yellow}33.09 &                 \cellcolor{yellow}0.961 &                    \cellcolor{yellow}0.043 &                    38.30 \\
\hline
Ours, PE                       & \cellcolor{orange}33.90 &                           \cellcolor{orange}0.965 &                    \cellcolor{orange}0.039  & \cellcolor{yellow}24.16 \\
Ours                           & \cellcolor{tablered}33.99 & \cellcolor{tablered}0.966 & \cellcolor{tablered}0.038 & \cellcolor{orange}23.22 \\
\end{tabular}
}
\vspace{-0.1in}
\caption{
Results for our method compared to previous approaches on the Blender dataset~\cite{mildenhall2020nerf}.
}
\label{tab:blendermean}
\end{table}

\begin{figure}
    \centering
    \begin{tabular}{@{}c@{\,}@{\,}c@{\,}|@{\,}c@{\,}@{\,}c@{\,}@{\,}c@{}}   & \scriptsize Ground truth   & \scriptsize Ours & \scriptsize Mip-NeRF~\cite{barron2021mipnerf} & \scriptsize VolSDF~\cite{yariv2021volume}\\
       & \scriptsize PSNR$\uparrow$/MAE$\downarrow$ & \scriptsize $\mathbf{33.6}$\textbf{dB}/$\mathbf{34.9^\circ}$ & \scriptsize $31.1$dB/$50.4^\circ$ & \scriptsize $22.8$dB/$40.3^\circ$\\
        \rotatebox{90}{\scriptsize \qquad\qquad\quad Normals \qquad\qquad\qquad\qquad\quad Rendering} &       
        \includegraphics[width=0.358\linewidth]{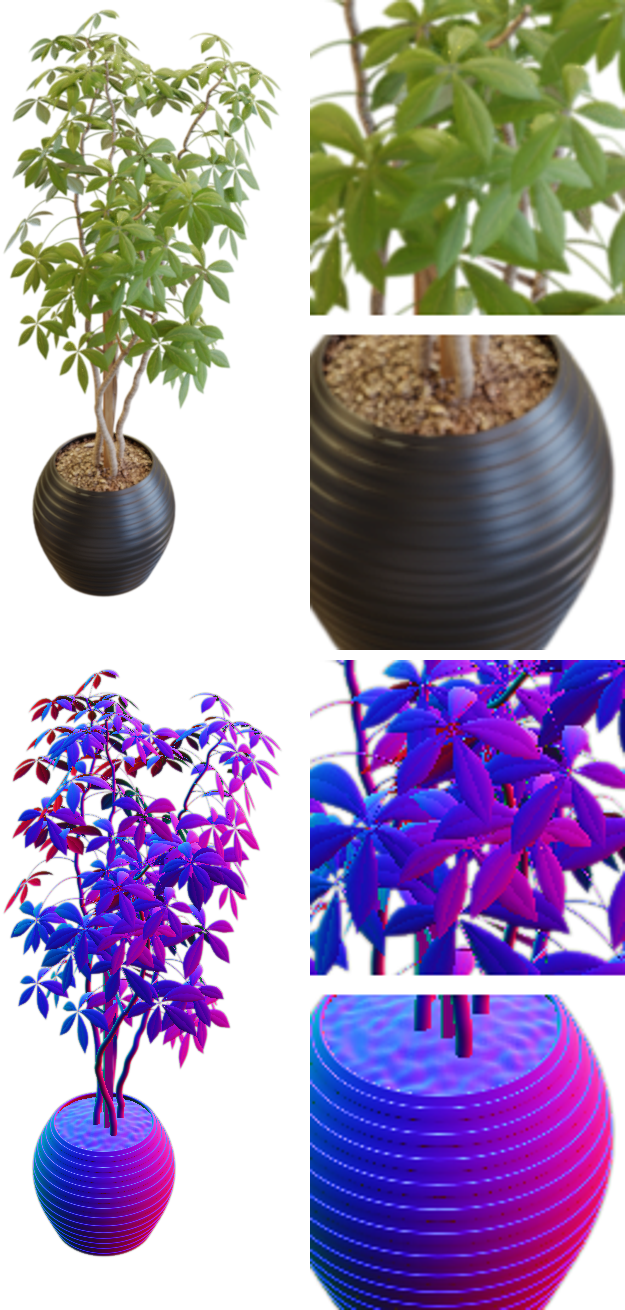} & 
        \includegraphics[width=0.18\linewidth]{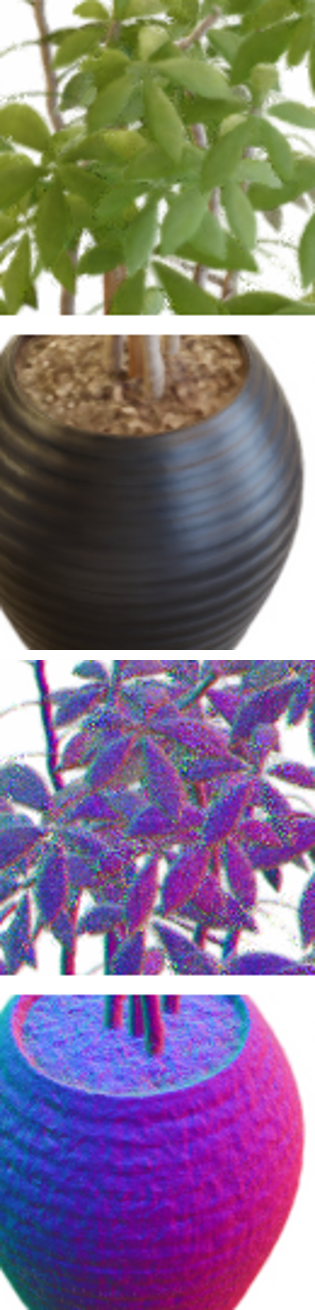} & 
        \includegraphics[width=0.18\linewidth]{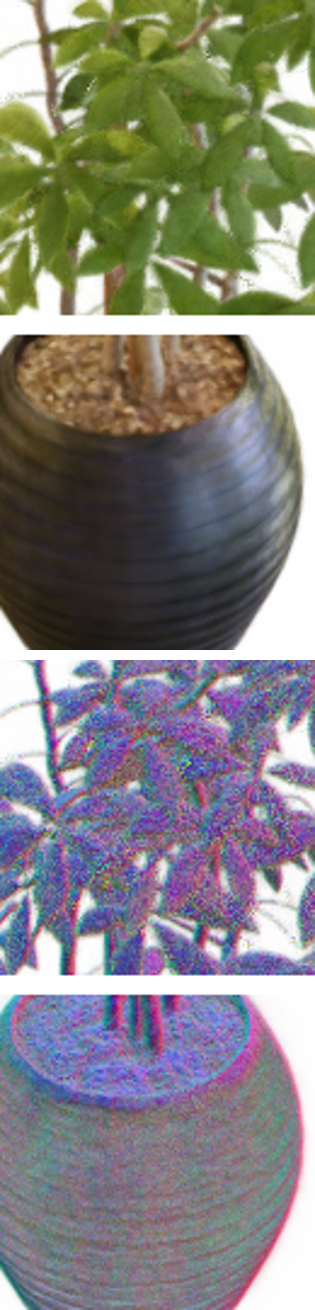} &
        \includegraphics[width=0.18\linewidth]{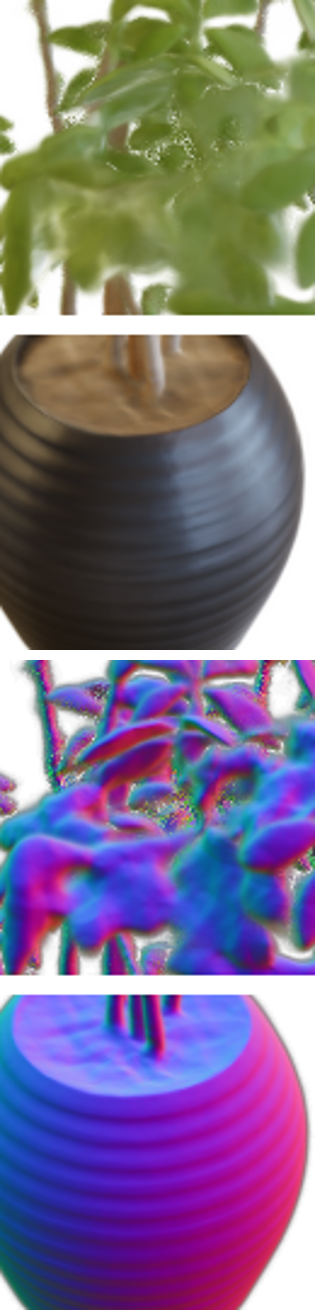}
        \\
    \end{tabular}
    \caption{Our model renders accurate glossy appearance and recovers fine geometric details. VolSDF~\cite{yariv2020idr} estimates accurate specularities and normal vectors but often fails to capture fine-scale details, such as leaves. Mip-NeRF~\cite{barron2021mipnerf} is able to capture fine structures but fails to faithfully render specular highlights (such as those on the pot and leaves) in novel views, and does not recover accurate normal vectors.
    } 
    \label{fig:ficus}
    \vspace{-0.1in}
\end{figure}

\begin{figure*}[t]
    \centering
    \includegraphics[width=\linewidth]{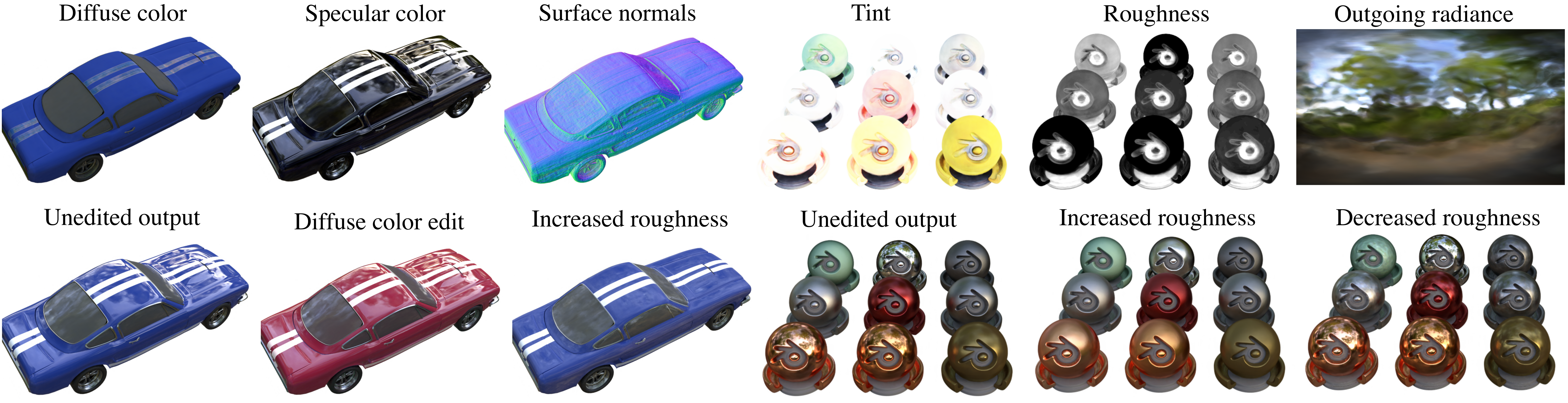}
    \caption{
    Our model's structuring of outgoing radiance decomposes scene appearance into interpretable components (top row) that enable editing (bottom row). Note that the recovered outgoing radiance function for a point on the chrome material ball (top right) is a plausible reconstruction of the actual scene lighting. We can edit the diffuse color of the car without affecting the specular reflections off its glossy paint, and we can plausibly modify the roughness of the car and material balls by manipulating the $\kappa$ values used in the IDE.
    }
    \label{fig:editing}
\end{figure*}

\begin{figure}
    \centering
    \begin{tabular}{c}
    \rotatebox{90}{\scriptsize \qquad\,\, Ground Truth \qquad\qquad\qquad\quad Ours \qquad\qquad\qquad\quad Mip-NeRF~\cite{barron2021mipnerf} } 
    \begin{overpic}[width=0.95\linewidth]{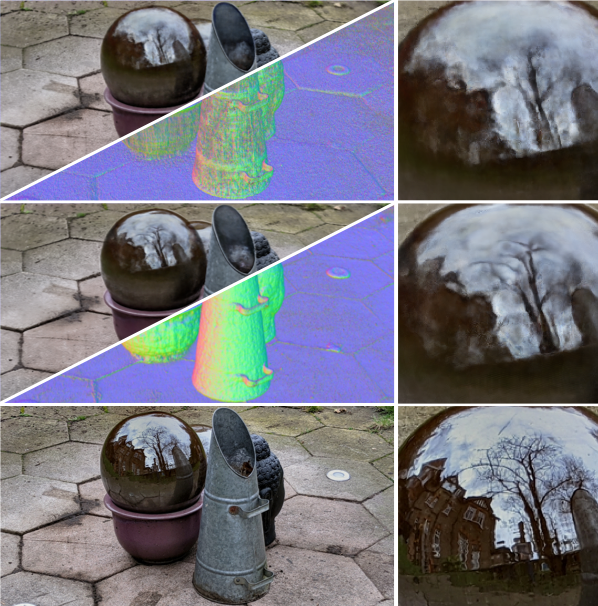} 
	\put(88.8, 67.8){\scriptsize \hl{$19.9$dB}}
	\put(88.1, 34.3){\scriptsize \hl{$\mathbf{20.8}$dB}}
    \end{overpic}
    \end{tabular}
    \caption{In this ``garden spheres'' scene, mip-NeRF's foggy geometry (see rendered normals) leads to blurred reflections (see inset, with PSNRs), while our model is able to recover accurate normal vectors and render more realistic reflections.
    } 
    \label{fig:real}
\end{figure}

\begin{figure}
    \centering
    \begin{tabular}{@{}c@{\,}@{\,}c@{\,}@{\,}c@{}}   \scriptsize Ground truth   & \scriptsize Ours & \scriptsize Mip-NeRF~\cite{barron2021mipnerf} \\
    \includegraphics[width=0.33\linewidth]{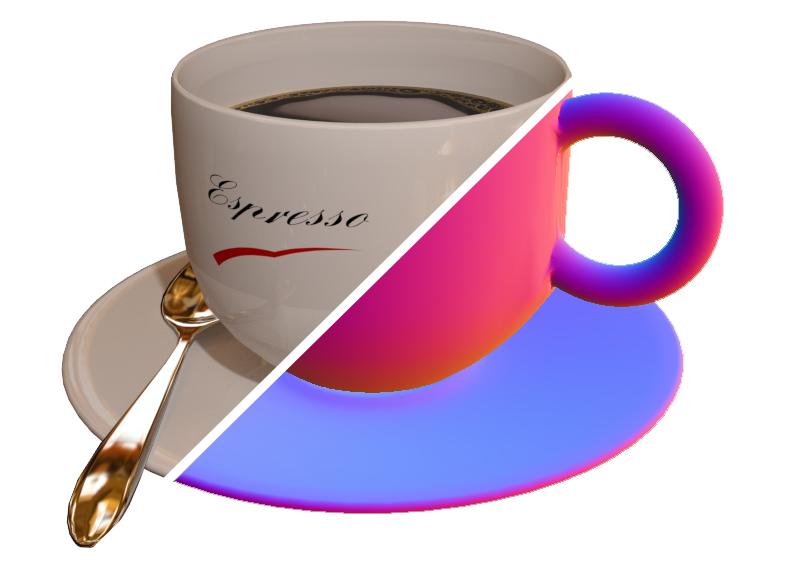} &
    \includegraphics[width=0.33\linewidth]{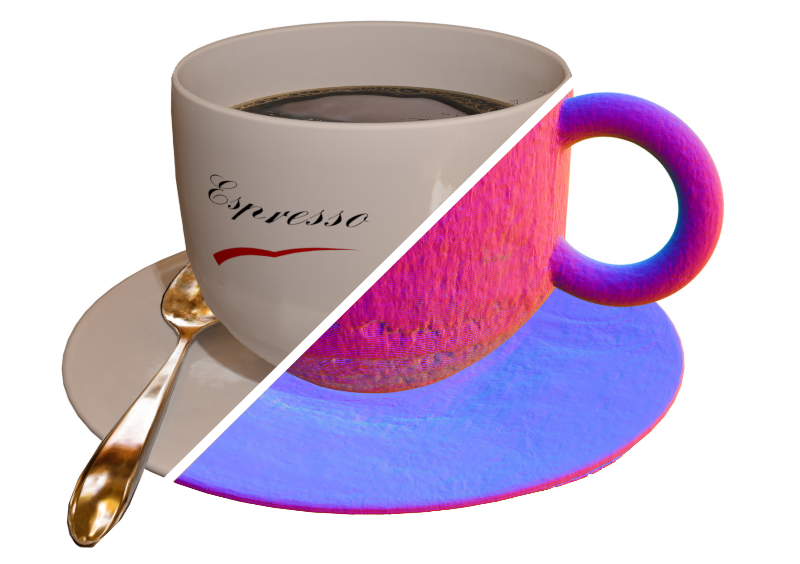} &
    \includegraphics[width=0.33\linewidth]{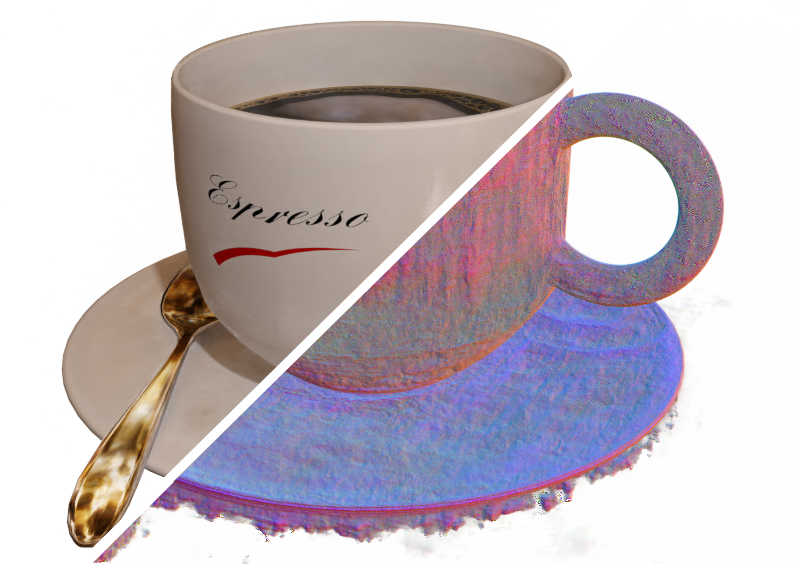}
    \end{tabular}
    \caption{On the ``coffee" scene from our ``Shiny Blender" dataset our method succeeds at estimating normals and interpolating specularities, whereas mip-NeRF~\cite{barron2021mipnerf} fails at doing both (see reflections on the spoon for example).} 
    \label{fig:coffee}
\end{figure}

\paragraph{Blender Dataset}

We also compare Ref-NeRF to recent neural view synthesis baseline methods on the standard Blender dataset from the original NeRF paper~\cite{mildenhall2020nerf}. Table~\ref{tab:blendermean} shows that our method outperforms all prior work across all image quality metrics. Our method also yields a large (35\%) improvement in the MAE of its normal vectors relative to mip-NeRF. While the hybrid surface-volume VolSDF representation~\cite{yariv2021volume} recovers slightly more accurate normal vectors (15\% lower MAE), our PSNR is substantially higher (6dB) than theirs. Additionally, VolSDF tends to oversmooth geometry, which makes our results qualitatively superior upon inspection (see Figure~\ref{fig:ficus}).

\paragraph{Real Captured Scenes}
In addition to these two synthetic datasets, we evaluate our model on a set of $3$ real captured scenes. We capture a ``sedan'' scene, and use the ``garden spheres'' and ``toy car'' captures from the Sparse Neural Radiance Grids paper~\cite{hedman2021snerg}. Figure~\ref{fig:real} and our supplement demonstrate that our rendered specular reflections and recovered normal vectors are often much more accurate on these real-world scenes.

\paragraph{Scene Editing}

Our structuring of outgoing radiance enables view-consistent editing of scenes. Although we do not perform a full inverse-rendering decomposition of appearance into BRDFs and lighting, our individual components behave intuitively and enable visually plausible scene editing results which are not attainable from a standard NeRF. Figure~\ref{fig:editing} shows example edits of the scene's components, and our supplementary video contains additional examples that demonstrate the view-consistency of our edited models.

\paragraph{Limitations}

While Ref-NeRF significantly improves upon previous top-performing neural scene representations for view synthesis, it requires increased computation: evaluating our integrated directional encoding is slightly slower than computing a standard positional encoding, and backpropagating through the gradient of the spatial MLP to compute normal vectors makes our model roughly 25\% slower than mip-NeRF. Our reparameterization of outgoing radiance by the reflection direction does not explicitly model interreflections or non-distant illumination, so our improvement upon mip-NeRF is reduced in such cases.

\section{Conclusion}
\label{sec:conclusion}

We have demonstrated that prior neural representations for view synthesis fail to accurately represent and render scenes with specularities and reflections. Our model, Ref-NeRF, introduces a new parameterization and structuring of view-dependent outgoing radiance, as well as a regularizer on normal vectors. These contributions allow Ref-NeRF to significantly improve both the quality of view-dependent appearance and the accuracy of normal vectors in synthesized views of the scene. We believe that this work makes important progress towards capturing and reproducing the rich photorealistic appearance of objects and scenes.

\paragraph{Acknowledgements}

We would like to thank Lior Yariv and Kai Zhang for helping us evaluate their methods, and Ricardo Martin-Brualla for helpful comments on our text. DV is supported by the National Science Foundation under Cooperative Agreement PHY-2019786 (an NSF AI Institute, \url{http://iaifi.org}).

{\small
\bibliographystyle{ieee_fullname}
\bibliography{egbib}
}

\clearpage

\appendix

\setcounter{equation}{0}
\setcounter{table}{0}
\setcounter{figure}{0}
\def\theequation{S\arabic{equation}}
\def\thetable{S\arabic{table}}
\def\thefigure{S\arabic{figure}}

\section{Integrated Directional Encoding Proofs}

We begin by proving the expression for the expected value of a spherical harmonic under a vMF distribution in Equation~7 in our main paper.
\begin{claim}
The expected value of a spherical harmonic function $Y_\ell^m(\uvomega)$ under a vMF distribution with mean $\refdir$ and concentration parameter $\kappa$ is:
\begin{equation}
\mathbb{E}_{\uvomega\sim\operatorname{vMF}(\refdir, \kappa)}[Y_\ell^m(\uvomega)] = A_\ell(\kappa)Y_\ell^m(\refdir),
\end{equation}
where:
\begin{equation} \label{eq:attenuationlegendre}
    A_\ell(\kappa) = \frac{\kappa}{2\sinh{\kappa}} \int_{-1}^1 P_\ell(u) e^{\kappa u}du,
\end{equation}
with $P_\ell$ the $\ell$th Legendre polynomial.
\end{claim}
\begin{proof}
We begin by first aligning the mean direction of the distribution $\refdir$ with the $z$-axis. We do this by applying a rotation matrix $R$ to transform $\uvomega' = R\uvomega$ where we choose $R$ that satisfies $R\refdir = \hat{z}$. Since the Jacobian of this transformation has unit determinant, we can write:
\begin{align} \label{eq:expdef}
\mathbb{E}_{\uvomega \sim \text{vMF}(\refdir, \kappa)}[Y_\ell^m(\uvomega)] &= c(\kappa)\int_{S^2} Y_\ell^m(\uvomega) e^{\kappa \refdir^\top \uvomega} d\uvomega \\
&= c(\kappa)\int_{S^2} Y_\ell^m(R^{-1}\uvomega) e^{\kappa \cos\theta}d\uvomega,\nonumber
\end{align}
where $c(\kappa) = \frac{\kappa}{4\pi \sinh{\kappa}}$ is the normalization factor of the vMF distribution, and $\theta$ is the angle between $\refdir$ and the $z$-axis.

A rotated spherical harmonic can be written as a linear combination of all spherical harmonics of the same degree, with coefficients specified by the Wigner D-matrix of the rotation:
\begin{align}
    Y_\ell^m(R^{-1}\uvomega) = \sum_{m'=-\ell}^\ell D_{mm'}^{(\ell)}(\refdir) Y_\ell^{m'}(\uvomega).
\end{align}

Plugging this into the expression from Equation~\ref{eq:expdef}:
\begin{align}
&\mathbb{E}_{\uvomega \sim \text{vMF}(\refdir, \kappa)}[Y_\ell^m(\uvomega)] \\
&=c(\kappa) \sum_{m'=-\ell}^\ell D_{mm'}^{(\ell)}(\refdir) \int_{S^2}   Y_\ell^{m'}(\uvomega) e^{\kappa \cos\theta}d\uvomega. \nonumber
\end{align}

The azimuthal dependence of the integrand is periodic in $2\pi$ for any $m'\neq 0$, so the only integral that does not vanish is the one with $m' = 0$, yielding:
\begin{align}
&\mathbb{E}_{\uvomega \sim \text{vMF}(\refdir, \kappa)}[Y_\ell^m(\uvomega)] \\
&= c(\kappa) D_{m0}^{(\ell)}(\refdir) \int_{S^2}   Y_\ell^{0}(\uvomega) e^{\kappa \cos\theta}d\uvomega. \nonumber
\end{align}

Plugging the known expression for the $\ell$th degree $0$th order spherical harmonic $Y_\ell^0(\uvomega) = \sqrt{\frac{2\ell+1}{4\pi}}P_\ell(\cos\theta)$ and the corresponding elements of the Wigner D-matrix $D_{m0}^{(\ell)}(\refdir) = \sqrt{\frac{4\pi}{2\ell+1}} Y_\ell^m(\refdir)$, and integrating over the azimuthal angle, we get:
\begin{align}
&\mathbb{E}_{\uvomega \sim \text{vMF}(\refdir, \kappa)}[Y_\ell^m(\uvomega)] \\
&= 2\pi c(\kappa) Y_\ell^m(\refdir) \int_0^\pi   P_\ell(\cos\theta) e^{\kappa \cos\theta} \sin\theta d\theta \nonumber \\
&= 2\pi c(\kappa) Y_\ell^m(\refdir) \int_{-1}^1 P_\ell(u) e^{\kappa u} du \nonumber \\
&= \frac{\kappa}{2\sinh{\kappa}} Y_\ell^m(\refdir) \int_{-1}^1 P_\ell(u) e^{\kappa u} du \nonumber \\
&= A_\ell(\kappa) Y_\ell^m(\refdir), \nonumber
\end{align}
where the second equality was obtained using the change of variables $u = \cos\theta$.
\end{proof}

Next, we show that the integral has a closed-form solution, leading to a simple expression for the $\ell$th attenuation function, $A_\ell(\kappa)$.
\begin{claim} \label{claim:attenuation}
For any $\ell \in \mathbb{N}$ the attenuation function $A_\ell$ satisfies:
\begin{equation} \label{eq:attenuationexact}
   A_\ell(\kappa) = \kappa^{-\ell} \sum_{i=0}^\ell \frac{(2\ell-i)!}{i!(\ell-i)!} \left(-2\right)^{i-\ell} b_i(\kappa),
\end{equation}
where $b_i(\kappa) = \kappa^{i}$ for even values of $i$ and $b_i(\kappa) = \kappa^{i}\coth(\kappa)$ for odd $i$.
\end{claim}
\begin{proof}
We prove this by first finding a recurrence relation for the attenuation functions, and then proving our expression by induction.

From Equation~\ref{eq:attenuationlegendre}, the $(\ell-1)$th attenuation function can be written as:
\begin{align}
A_{\ell-1}(\kappa) &= \frac{\kappa}{2\sinh\kappa} \int_{-1}^1 P_{\ell-1}(u)e^{\kappa u}du \\
&= \frac{\kappa}{2\sinh\kappa} \int_{-1}^1 \left(\frac{d}{du}\frac{P_\ell(u) - P_{\ell-2}(u)}{2\ell-1}\right)e^{\kappa u}du \nonumber\\
&= \frac{\kappa}{2\sinh\kappa} \left(\frac{P_\ell(u) - P_{\ell-2}(u)}{2\ell-1}\right)e^{\kappa u} \Big|_{-1}^1 \nonumber\\
&\quad- \frac{\kappa^2}{2\sinh\kappa} \int_{-1}^1 \left(\frac{P_\ell(u) - P_{\ell-2}(u)}{2\ell-1}\right)e^{\kappa u}du \nonumber\\
&= -\frac{\kappa}{2\ell-1} \left(A_\ell(\kappa) - A_{\ell-2}(\kappa)\right). \nonumber
\end{align}
where the second equality was obtained using a known recurrence relation of the Legendre polynomials, the third was obtained using integration by parts, and the fourth by using the fact that $P_\ell(\pm1) - P_{\ell-2}(\pm1)$ = 0. Reordering, we get:
\begin{equation} \label{eq:recurrence}
    A_\ell(\kappa) = A_{\ell-2}(\kappa) - \frac{2\ell-1}{\kappa}A_{\ell-1}(\kappa).
\end{equation}

We can easily find the first two attenuation functions by directly computing the integrals:
\begin{align} \label{eq:explicitl01}
    A_{0}(\kappa) &= \frac{\kappa}{2\sinh\kappa} \int_{-1}^1 e^{\kappa u}du = 1, \\
    A_{1}(\kappa) &= \frac{\kappa}{2\sinh\kappa} \int_{-1}^1 u e^{\kappa u}du = \coth(\kappa) - \frac{1}{\kappa}.\nonumber
\end{align}

Finally, we prove our claim by induction using the recurrence relation in Equation~\ref{eq:recurrence}. It is easy to verify that our expression holds for $\ell=0$ and $\ell=1$ by plugging these values in Equation~\ref{eq:attenuationexact} and comparing with Equation~\ref{eq:explicitl01}.

We now assume that the relation holds for any $\ell-2$ and $\ell-1$ and prove it for $\ell \geq 2$. We can do this by considering the coefficient of $\kappa^{-m}$ for every $m \in \{0, ..., \ell\}$ of the right hand side of Equation~\ref{eq:recurrence}, and show that it is identical to the one from Equation~\ref{eq:attenuationexact}. We begin by separately considering $m=0$, $m=\ell-1$ and $m=\ell$, and then any other $m$. Note that for some $m$ values $\kappa^{-m}$ is multiplied by $\coth(\kappa)$, but this factor always matches in $A_{\ell-2}(\kappa)$, in $\frac{1}{\kappa}A_{\ell-1}(\kappa)$, and in $A_\ell(\kappa)$, so we neglect it.

For $m=0$, the only contribution to the coefficient of $\kappa^0$ is from $A_{\ell-2}$, which gives a coefficient of $1$, matching with the $\kappa^0$ coefficient of $A_\ell$.

For $m=\ell-1$ we only get a contribution from the second term, and the coefficient is:
\begin{align}
-(2\ell-1)&\frac{(2\ell-3)!}{1!(\ell-2)!}(-2)^{2-\ell} \\
&=\frac{(2\ell-1)(\ell-1)(2\ell-3)!}{1!(\ell-1)(\ell-2)!}\cdot 2 \cdot (-2)^{1-\ell} \nonumber\\
&=\frac{(2\ell-1)(2\ell-2)(2\ell-3)!}{1!(\ell-1)(\ell-2)!} (-2)^{1-\ell} \nonumber\\
&=\frac{(2\ell-1)!}{1!(\ell-1)!} (-2)^{1-\ell}, \nonumber
\end{align}
which matches with that of $A_\ell$ from Equation~\ref{eq:attenuationexact}.

Similarly, for $m=\ell$ the contribution is also from the second term, with the coefficient:
\begin{align}
-(2\ell-1)&\frac{(2\ell-2)!}{0!(\ell-1)! (-2)^{1-\ell}} \\
&= 2 \frac{(2\ell-1)!}{0!(\ell-1)!}(-2)^\ell \nonumber \\
&= 2 \frac{2\ell (2\ell-1)!}{0!\cdot 2\ell \cdot (\ell-1)!}(-2)^\ell \nonumber \\
&= \frac{(2\ell)!}{0!\ell!}(-2)^{-\ell} \nonumber,
\end{align}
and again this matches with the coefficient of $A_\ell$.

Finally, for any $m \in \{1, ..., \ell-2\}$, we have that the coefficient of $\kappa^{-m}$ on the right hand side of Equation~\ref{eq:recurrence} is:
\begin{align}
&\frac{(\ell+m-2)!}{(\ell-m-2)!m!}(-2)^{-m} \\
&\quad- (2\ell-1)\frac{(\ell+m-2)!}{(\ell-m)!(m-1)!}(-2)^{1-m} \nonumber \\
&= \frac{(-2)^{-m} (\ell+m-2)!}{(\ell-m)!m!} \nonumber \\
&\quad\cdot\left[(\ell-m)(\ell-m-1)+2m(2\ell-1)\right] \nonumber \\
&= \frac{(-2)^{-m} (\ell+m-2)!}{(\ell-m)!m!}(\ell+m)(\ell+m-1) \nonumber \\
&= \frac{(\ell+m)!}{(\ell-m)!m!}(-2)^{-m}, \nonumber
\end{align}
which is identical to the coefficient in Equation~\ref{eq:attenuationexact}, finishing our proof.

\end{proof}

Computing the attenuation function using Equation~\ref{eq:attenuationexact} is inefficient, and it is numerically unstable due to catastrophic cancellation. In Equation~8 of the main paper we present an approximation which is guaranteed to be close to the exact functions from Equation~\ref{eq:attenuationexact}, for large values of $\kappa$. We finish this section with a proof of this claim. Additionally, Figure~\ref{fig:attenuation_approximation} shows that our approximation closely resembles the exact attenuation functions for all values of $\kappa > 0$ and all orders $\ell$, and that its quality improves as $\ell$ increases.

\begin{claim}
For large values of $\kappa$, our approximation is exact up to an $O(\nicefrac{1}{\kappa^2})$ term, \ie:
\begin{equation}
A_\ell(\kappa) = \exp\left(-\frac{\ell(\ell+1)}{2\kappa}\right) + O\left(\frac{1}{\kappa^2}\right)
\end{equation}
\end{claim}
\begin{proof}
Using the fact that $\coth(\kappa) = 1 + O(e^{-2\kappa})$, and using  the $i=\ell$ and $i=\ell-1$ terms from the sum in Equation~\ref{eq:attenuationexact} of Claim~\ref{claim:attenuation}, we have that:
\begin{equation}
    A_\ell(\kappa) = 1 - \frac{\ell(\ell+1)}{2\kappa} + O\left(\frac{1}{\kappa^2}\right).
\end{equation}

Using a 1st order Taylor approximation for the exponential function yields:
\begin{equation}
\exp\left(-\frac{\ell(\ell+1)}{2\kappa}\right) = 1 - \frac{\ell(\ell+1)}{2\kappa} + O\left(\frac{1}{\kappa^2}\right),
\end{equation}
proving our claim.
\end{proof}

\begin{figure}
    \centering
    \includegraphics[width=\linewidth]{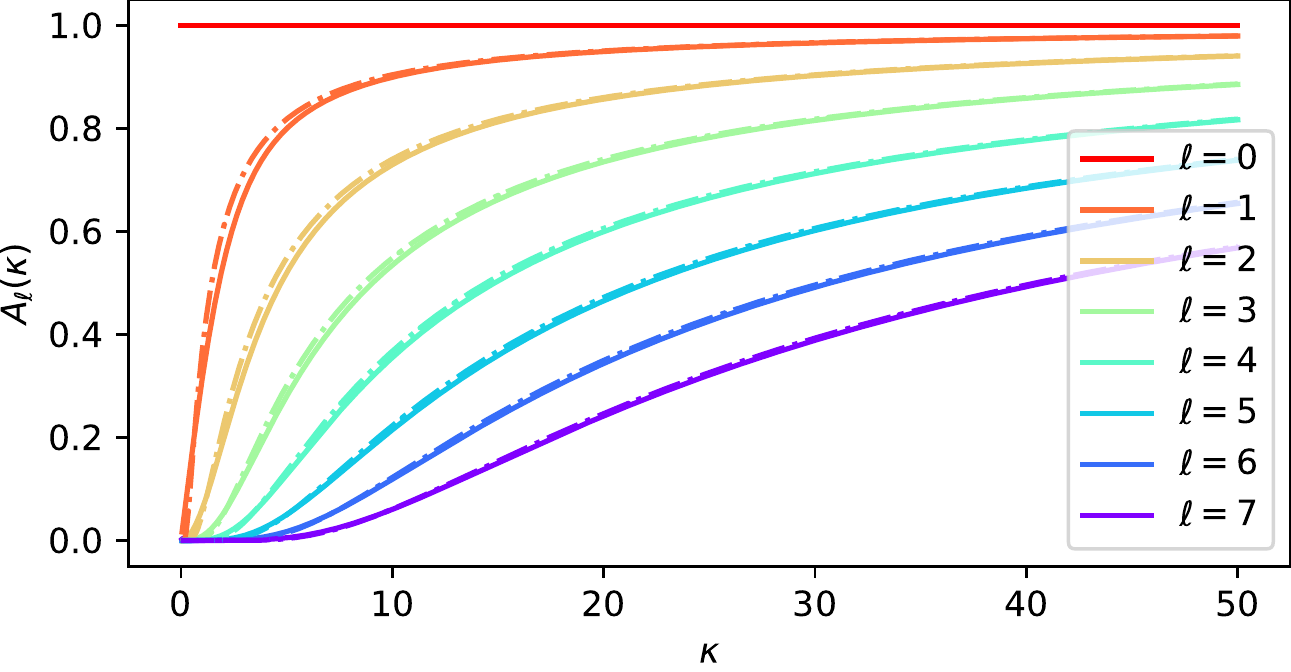}
    \caption{The first eight attenuation functions $A_\ell(\kappa)$. The exact expression (in solid lines) matches the approximation (in dashed lines), with an increasing accuracy as $\ell$ increases.}
    \label{fig:attenuation_approximation}
    \vspace{-0.2in}
\end{figure}

\section{Interpretations of Normal Vector Penalty}

As described in the main paper, the regularization term in Equation~11 penalizes ``backwards-facing'' normal vectors that contribute to the ray's final rendered color. 

Note that Equation~11 applies this penalty to the normals output by the spatial MLP as opposed to the normal vectors directly computed from the gradient of the density field. While this penalty is remarkably effective at improving recovered normal vectors, we find that directly applying it to the gradient density normals sometimes affects the optimization dynamics for fine geometric structures. Even when NeRF recovers a concentrated surface at the end of optimization, it typically passes through a fuzzier volumetric representation during optimization, and directly penalizing the gradient of the volume density can sometimes adversely affect this trajectory. Our approach of applying the orientation penalty to the predicted normals gives the spatial MLP two options at every location: it can either predict normals that agree with the density gradient ones, in which case our orientation loss is effectively applied to the density field; or it can predict normals that deviate from those computed from the density field, and incur a normal prediction penalty $\mathcal{R}_p$. We find that applying the penalty to predicted normals allows our method to adaptively apply the orientation loss, resulting in accurate normals without loss of fine details (see ablation studies). 

Besides providing smooth normals for computing reflection directions and providing a mechanism for an adaptive corner-preserving orientation loss, the predicted normals also allow the model to directly encode view directions by predicting a constant normal direction $\uvn'(\vx) = \uvc$, leading to a 1-to-1 mapping from view direction to reflection direction (see Equation~4). While this is discouraged by the normal prediction loss $\mathcal{R}_p$, our model can exploit this behavior in regions that are not well-described by a surface at a scale whose geometry the density field can capture.

\section{Optimization}

Our implementation is based on the official JAX~\cite{jax2018github} implementation of mip-NeRF~\cite{barron2021mipnerf}.

In all of our experiments on synthetic data we apply the normal orientation loss from Equation~11 and the normal prediction loss from Equation~10 of the main paper with weights of $0.1$ and $3\cdot10^{-4}$ respectively, relative to the standard NeRF data loss from Equation~2. In our experiments on real captured data, we apply a slightly higher weight of $10^{-3}$ on the normal prediction loss. We use the same weights during both the coarse and fine stages.

During training, we add i.i.d.\ Gaussian noise with standard deviation $0.1$ to the bottleneck $\vb$. We find that this slightly stabilizes our model's results in some cases by preventing it from using the bottleneck early in training. 

We optimize all experiments on the Blender scenes using single-image batching to be consistent with results reported in the NeRF and mip-NeRF papers. For our Shiny Blender dataset, we sample random batches of rays from all images. 

We train our model, ablations of our model, and all mip-NeRF baselines using a slightly modified version of mip-NeRF's learning schedule: $250\mathrm{k}$ optimization iterations with a batch size of $2^{14}$, using the Adam~\cite{adam} optimizer with hyperparameters $\beta_1 = 0.9$, $\beta_2= 0.999$, $\varepsilon=10^{-6}$, a learning rate that is annealed log-linearly from $2 \times 10^{-3}$ to $2 \times 10^{-5}$ with a warm-up phase of $512$ iterations, and gradient clipping to a norm of $10^{-3}$.

\section{Dataset Details}

Our new ``Shiny Blender" scenes were adapted from the the following BlendSwap models:
\begin{compactenum}
    \item Coffee: created by \emph{sleem}, CC-0 license (model \#10827).
    \item Toaster: created by \emph{PrinterKiller}, CC-BY license (model \#4989)
    \item Car: created by \emph{Xali}, CC-0 license (model \#24359).
    \item Helmet: created by \emph{kveidem}, CC-0 license (model \#21617).
\end{compactenum}

Our dataset of real scenes was captured by the authors. 

\section{Evaluation Details}

\paragraph{Normal Vectors}

To compute the normal vector for a ray, we sample normals along the ray $\{\uvn_i\}$ using Equation~3 from the main paper, and use the volume rendering procedure from Equation~1 to compute a single normal vector:
\begin{equation} \label{eq:accnormals}
    \vN(\vo, \uvd) = \sum_i w_i \uvn(\vx_i).
\end{equation}

We use Equation~\ref{eq:accnormals} to visualize normal maps, with gray-colored values corresponding to high variance normal vectors along a ray. For evaluating MAE, we use the normalized accumulated normals, $\uvN = \vN / \|\vN\|$.

\begin{table}[]
\centering
\resizebox{\linewidth}{!}{
\begin{tabular}{@{}l@{\,\,}|@{\,\,}ccc@{}}
& \textit{sedan} & \textit{toy car} & \textit{garden spheres} \\ \hline
Mip-NeRF, 8 layers     & 25.53 / \textbf{0.729} / \textbf{0.118} & 24.00 / 0.663 / 0.177 & 23.40 / \textbf{0.620} / \textbf{0.125} \\
Ours                & \textbf{25.65} / 0.720 / 0.119 & \textbf{24.25} / \textbf{0.674} / \textbf{0.168} & \textbf{23.46} / 0.601 / 0.138 
\end{tabular}
}
\vspace{-0.1in}
\caption{
Quantitative metrics (PSNR/SSIM/LPIPS) on the test sets of our real captured scenes.
}
\label{tab:real}
\vspace{-0.1in}
\end{table}
 
\paragraph{Baseline Implementations}

For comparisons to mip-NeRF~\cite{barron2021mipnerf}, we use the official open-source JAX implementation. For comparisons to PhySG~\cite{zhang2021physg}, we use the official PyTorch code open-sourced by the authors. However, the original hyperparameters produce poor results on our datasets, so the PhySG authors helped us tune their hyperparameters for the Blender datasets. VolSDF~\cite{yariv2021volume} does not currently have publicly-available code, but the authors graciously ran their code on the Blender dataset and provided us with rendered test-set images and normals. We report the NSVF~\cite{liu2020nsvf} and NeRF~\cite{mildenhall2020nerf} results on the Blender dataset from the tables in the mip-NeRF paper.

\section{Additional Results}

Table~\ref{tab:real} reports quantitative metrics on the test sets (every eighth image is held out for testing, as done in NeRF~\cite{mildenhall2020nerf}) of our three real captured scenes. Tables~\ref{tab:psnrs_allscenes},~\ref{tab:ssims_allscenes},~\ref{tab:lpips_allscenes}, and~\ref{tab:maes_allscenes} contain quantitative metrics on the original synthetic Blender dataset, and Tables~\ref{tab:shiny_psnrs_allscenes},~\ref{tab:shiny_ssims_allscenes},~\ref{tab:shiny_lpips_allscenes}, and~\ref{tab:shiny_maes_allscenes} contain quantitative metrics on our new synthetic Shiny Blender dataset.

\begin{table}[]
    \centering
    \resizebox{\linewidth}{!}{
    \begin{tabular}{@{}l|ccccccccr}
     & \!chair\! & \!lego\! & \!materials\! & \!mic\! & \!hotdog\! & \!ficus\! & \!drums\! & \!ship\!  \\ \hline
     PhySG~\cite{zhang2021physg}    &                    24.00 &                    20.19 &                    18.86 &                    22.33 &                    24.08 &                    19.02 &                    20.99 &                    15.35 \\
    VolSDF~\cite{yariv2021volume}   &                    30.57 &                    29.46 &                    29.13 &                    30.53 &                    35.11 &                    22.91 &                    20.43 &                    25.51 \\
    Mip-NeRF~\cite{barron2021mipnerf} & \cellcolor{yellow} 35.12 & \cellcolor{yellow} 35.92 & \cellcolor{yellow} 30.64 & \cellcolor{tablered}    36.76 & \cellcolor{yellow} 37.34 & \cellcolor{yellow} 33.19 & \cellcolor{yellow} 25.36 & \cellcolor{tablered}    30.52 \\ \hline
    Ours, PE & \cellcolor{tablered}    35.86 & \cellcolor{tablered}    36.33 & \cellcolor{orange} 35.22 & \cellcolor{yellow} 35.84 & \cellcolor{tablered}    38.18 & \cellcolor{orange} 33.60 & \cellcolor{tablered}    26.03 & \cellcolor{yellow} 30.17 \\
    Ours     & \cellcolor{orange} 35.83 & \cellcolor{orange} 36.25 & \cellcolor{tablered}    35.41 & \cellcolor{tablered}    36.76 & \cellcolor{orange} 37.72 & \cellcolor{tablered}    33.91 & \cellcolor{orange} 25.79 & \cellcolor{orange} 30.28 \\
    \end{tabular}
    }
    \vspace{-0.1in}
    \caption{
    Per-scene test set PSNRs on the Blender dataset~\cite{mildenhall2020nerf}.
    }
    \label{tab:psnrs_allscenes}

    \resizebox{\linewidth}{!}{
    \begin{tabular}{@{}l|ccccccccr}
     & \!chair\! & \!lego\! & \!materials\! & \!mic\! & \!hotdog\! & \!ficus\! & \!drums\! & \!ship\!  \\ \hline
     PhySG~\cite{zhang2021physg}    &   0.898 &                            0.821 &          0.838 &    0.933 &   0.912 &   0.873 &     0.884 &     0.727 \\
    VolSDF~\cite{yariv2021volume}   &  0.949 &                             0.951 &                                        0.954 & 0.969 & 0.972 &  0.929 &  0.893 &  0.842 \\
    Mip-NeRF~\cite{barron2021mipnerf} & \cellcolor{yellow} 0.981 &     \cellcolor{yellow}    0.980 & \cellcolor{yellow}    0.959 &  \cellcolor{tablered} 0.992 &  \cellcolor{yellow} 0.982 &  \cellcolor{yellow} 0.980 & \cellcolor{yellow} 0.933 &   \cellcolor{tablered}  0.885 \\ \hline
    Ours, PE & \cellcolor{tablered}             0.984 &               \cellcolor{tablered}     0.981 & \cellcolor{tablered} 0.983 & \cellcolor{yellow} 0.991 &  \cellcolor{tablered} 0.984 & \cellcolor{orange} 0.982 &    \cellcolor{tablered}                0.939 & \cellcolor{yellow} 0.878 \\
    Ours     &  \cellcolor{tablered}            0.984 &               \cellcolor{tablered}     0.981 & \cellcolor{tablered}   0.983 & \cellcolor{tablered} 0.992 & \cellcolor{tablered} 0.984 &  \cellcolor{tablered} 0.983 &  \cellcolor{orange}  0.937 & \cellcolor{orange}  0.880     \end{tabular}
    }
    \vspace{-0.1in}
    \caption{
    Per-scene test set SSIMs on the Blender dataset~\cite{mildenhall2020nerf}.
    }
    \label{tab:ssims_allscenes}

    \resizebox{\linewidth}{!}{
    \begin{tabular}{@{}l|ccccccccr}
     & \!chair\! & \!lego\! & \!materials\! & \!mic\! & \!hotdog\! & \!ficus\! & \!drums\! & \!ship\!  \\ \hline
     PhySG~\cite{zhang2021physg}    &      0.093 &      0.172 &      0.142 &      0.082 &      0.117 &      0.112 &      0.113 &      0.322 \\
    VolSDF~\cite{yariv2021volume}   &               0.056 &                 0.054 &                    0.048 &                     0.191 &                       0.043 &                    0.068 &                      0.119 &                    0.191 \\
    Mip-NeRF~\cite{barron2021mipnerf} & \cellcolor{yellow} 0.020 & \cellcolor{tablered}    0.018 & \cellcolor{yellow} 0.040 & \cellcolor{orange} 0.008 & \cellcolor{yellow} 0.026 & \cellcolor{yellow} 0.021 & \cellcolor{yellow} 0.064 & \cellcolor{tablered}    0.135 \\ \hline
    Ours, PE & \cellcolor{tablered}    0.017 & \cellcolor{tablered}    0.018 & \cellcolor{orange} 0.023 & \cellcolor{orange} 0.008 & \cellcolor{tablered}    0.022 & \cellcolor{orange} 0.020 & \cellcolor{tablered}    0.059 & \cellcolor{yellow} 0.143 \\
    Ours     & \cellcolor{tablered}    0.017 & \cellcolor{tablered}    0.018 & \cellcolor{tablered}    0.022 & \cellcolor{tablered}    0.007 & \cellcolor{tablered}    0.022 & \cellcolor{tablered}    0.019 & \cellcolor{tablered}    0.059 & \cellcolor{orange} 0.139 \\
    \end{tabular}
    }
    \vspace{-0.1in}
    \caption{
    Per-scene test set LPIPS on the Blender dataset~\cite{mildenhall2020nerf}.
    }
    \label{tab:lpips_allscenes}

    \resizebox{\linewidth}{!}{
    \begin{tabular}{@{}l|ccccccccr}
     & \!chair\! & \!lego\! & \!materials\! & \!mic\! & \!hotdog\! & \!ficus\! & \!drums\! & \!ship\!  \\ \hline
     PhySG~\cite{zhang2021physg} &     \cellcolor{orange} 18.569 &      40.244 &      18.986 &      26.053 &      28.572 &      \cellcolor{tablered} 35.974 &    \cellcolor{orange}  21.696 &      43.265 \\
    VolSDF~\cite{yariv2021volume}   & \cellcolor{tablered}    14.085 & \cellcolor{yellow} 26.619 & \cellcolor{tablered}    8.277 & \cellcolor{tablered}    19.579 & \cellcolor{tablered}    12.170 & \cellcolor{orange}    39.801 & \cellcolor{tablered}    21.458 & \cellcolor{tablered}    16.974 \\
    Mip-NeRF~\cite{barron2021mipnerf} &                    28.044 &                    30.532 &                    64.074 &                    36.489 &                    29.303 &                    53.524 &                    32.374 &                    37.667 \\ \hline
    Ours, PE &  20.018 & \cellcolor{orange} 26.471 & \cellcolor{yellow} 10.162 & \cellcolor{yellow} 25.921 & \cellcolor{yellow} 13.920 &  41.557 & \cellcolor{yellow} 27.766 & \cellcolor{yellow} 34.212 \\
    Ours     & \cellcolor{yellow} 19.852 & \cellcolor{tablered}    24.469 & \cellcolor{orange} 9.531 & \cellcolor{orange} 24.938 & \cellcolor{orange} 13.211 & \cellcolor{yellow} 41.052 &  27.853 & \cellcolor{orange} 31.707 \\
    \end{tabular}
    }
    \vspace{-0.1in}
    \caption{
    Per-scene test set normal MAEs on the Blender dataset~\cite{mildenhall2020nerf}.
    }
    \label{tab:maes_allscenes}
    \vspace{-0.2in}
\end{table}

\begin{table}[]
    \centering
    \resizebox{\linewidth}{!}{
    \begin{tabular}{@{}l|ccccccr}
     & \!teapot\! & \!toaster\! & \!car\! & \!ball\! & \!coffee\! & \!helmet\!  \\ \hline
     PhySG~\cite{zhang2021physg} &      35.83 &      18.59 &      24.40 &      27.24 &      23.71 &      27.51 \\ \hline
     Mip-NeRF~\cite{barron2021mipnerf}                       &                    46.00 &                    22.37 &                    26.50 &                    25.94 &                    30.36 &                    27.39 \\
    Mip-NeRF, 8 layers             & \cellcolor{orange} 47.35 &                    25.51 &                    27.99 &                    27.53 &                    32.14 &                    29.04 \\
    Mip-NeRF, 8 layers, w/ normals &                    47.09 &                    25.14 &                    27.97 &                    26.79 &                    32.12 &                    29.21 \\
    Mip-NeRF, 8 layers,  with $\mathcal{R}_o$ & \cellcolor{orange} 47.35 &                    25.32 &                    27.91 &                    26.89 &                    32.21 &                    29.22 \\ \hline
    Ours, no reflection             &                    44.74 &                    24.04 &                    27.41 &                    20.94 &                    31.95 &                    27.76 \\
    Ours, no $\mathcal{R}_o$          &                    46.80 & \cellcolor{orange} 25.78 &                    28.43 &                    27.06 &                    32.58 &                    29.06 \\
    Ours, no pred. normals         &                     47.09 &                    23.32 &                    27.19 &                    26.09 &                    31.79 & \cellcolor{tablered}    30.54 \\
    Ours, concat. viewdir          &                    46.01 &                    25.38 & \cellcolor{orange} 30.71 &                    47.45 &                    34.19 &                    28.81 \\
    Ours, fixed lobe               &                    46.82 &                    25.57 &                    30.09 &                    47.25 & \cellcolor{orange} 34.37 &                    29.00 \\ \hline
    Ours, no diffuse color         &                    46.45 &                    25.56 &                    30.46 &                    34.53 &                    34.05 &                    28.87 \\
    Ours, no tint                  &                    46.54 &                    25.49 &                    30.14 & \cellcolor{orange} 47.53 & \cellcolor{yellow} 34.24 &                    28.78 \\
    Ours, no roughness             &                    45.28 &                    25.39 &                    30.44 &                    36.33 &                    33.19 & \cellcolor{orange} 29.72 \\
    Ours, PE                       &                    46.55 & \cellcolor{tablered}    26.74 & \cellcolor{yellow} 30.53 & \cellcolor{tablered}    47.56 & \cellcolor{tablered}    34.45 &                    29.59 \\
    Ours                           & \cellcolor{tablered}    47.90 & \cellcolor{yellow} 25.70 & \cellcolor{tablered}    30.82 & \cellcolor{yellow} 47.46 &                    34.21 & \cellcolor{yellow} 29.68 \\
    \end{tabular}
    }
    \vspace{-0.1in}
    \caption{
    Per-scene test set PSNRs on our Shiny Blender dataset.
    }
    \label{tab:shiny_psnrs_allscenes}

    \resizebox{\linewidth}{!}{
    \begin{tabular}{@{}l|ccccccr}
     & \!teapot\! & \!toaster\! & \!car\! & \!ball\! & \!coffee\! & \!helmet\!  \\ \hline
     PhySG~\cite{zhang2021physg} &      0.990 &      0.805 &      0.910 &      0.947 &      0.922 &      0.953 \\ \hline
     Mip-NeRF~\cite{barron2021mipnerf}                       & \cellcolor{orange} 0.997 &                    0.891 &                    0.922 &                    0.935 &                    0.966 &                    0.939 \\
Mip-NeRF, 8 layers             & \cellcolor{orange} 0.997 & \cellcolor{yellow} 0.925 &                    0.936 &                    0.949 &                    0.971 &                    0.957 \\
Mip-NeRF, 8 layers, w/ normals & \cellcolor{orange} 0.997 &                    0.923 &                    0.936 &                    0.946 &                    0.970 &                    0.957 \\
Mip-NeRF, 8 layers,  with $\mathcal{R}_o$ & \cellcolor{orange} 0.997 &                    0.924 &                    0.935 &                    0.947 &                    0.971 &                      0.957 \\ \hline
Ours, no reflection            &                    0.996 &                    0.912 &                    0.930 &                    0.905 &                    0.970 &                    0.950 \\
Ours, no $\mathcal{R}_o$       & \cellcolor{orange} 0.997 & \cellcolor{orange} 0.926 &                    0.937 &                    0.939 &                    0.971 &                    0.956 \\
Ours, no pred. normals         & \cellcolor{orange} 0.997 &                    0.898 &                    0.926 &                    0.865 &                    0.967 & \cellcolor{tablered}    0.962 \\
Ours, concat. viewdir          & \cellcolor{orange} 0.997 &                    0.919 & \cellcolor{tablered}    0.956 & \cellcolor{orange} 0.995 & \cellcolor{orange} 0.974 &                    0.952 \\
Ours, fixed lobe               & \cellcolor{orange} 0.997 &                    0.920 &                    0.952 & \cellcolor{orange} 0.995 & \cellcolor{orange} 0.974 &                    0.954 \\ \hline
Ours, no diffuse color         & \cellcolor{orange} 0.997 &                    0.920 &                    0.953 &                    0.977 &                    0.973 &                    0.954 \\
Ours, no tint                  & \cellcolor{orange} 0.997 &                    0.921 &                    0.951 & \cellcolor{orange} 0.995 & \cellcolor{orange} 0.974 &                    0.954 \\
Ours, no roughness             &                    0.996 &                    0.917 & \cellcolor{yellow} 0.954 &                    0.983 &                    0.972 & \cellcolor{orange} 0.958 \\
Ours, PE                       & \cellcolor{orange} 0.997 & \cellcolor{tablered}    0.928 & \cellcolor{yellow} 0.954 & \cellcolor{tablered}    0.996 & \cellcolor{tablered}    0.975 & \cellcolor{orange} 0.958 \\
Ours                           & \cellcolor{tablered} 0.998 &                    0.922 & \cellcolor{orange} 0.955 & \cellcolor{orange} 0.995 & \cellcolor{orange} 0.974 & \cellcolor{orange} 0.958 \\
    \end{tabular}
    }
    \vspace{-0.1in}
    \caption{
    Per-scene test set SSIMs on our Shiny Blender dataset.
    }
    \label{tab:shiny_ssims_allscenes}

    \resizebox{\linewidth}{!}{
    \begin{tabular}{@{}l|ccccccr}
     & \!teapot\! & \!toaster\! & \!car\! & \!ball\! & \!coffee\! & \!helmet\!  \\ \hline
     PhySG~\cite{zhang2021physg} &      0.022 &      0.194 &      0.091 &      0.179 &      0.150 &      0.089 \\ \hline
     Mip-NeRF~\cite{barron2021mipnerf}                       &                    0.008 &                    0.123 &                    0.059 &                    0.168 &                    0.086 &                    0.108 \\
Mip-NeRF, 8 layers             & \cellcolor{orange} 0.006 & \cellcolor{tablered}    0.080 &                    0.052 &                    0.139 &                    0.082 & \cellcolor{orange} 0.072 \\
Mip-NeRF, 8 layers, w/ normals & \cellcolor{orange} 0.006 & \cellcolor{yellow} 0.085 &                    0.052 &                    0.144 &                    0.082 & \cellcolor{orange} 0.072 \\
Mip-NeRF, 8 layers,  with $\mathcal{R}_o$ & \cellcolor{orange} 0.006 & \cellcolor{orange} 0.082 &                    0.052 &                    0.143 &                    0.082 & \cellcolor{orange} 0.072 \\ \hline
Ours, no reflection             &                   0.007 &                    0.091 &                    0.052 &                    0.192 &                    0.082 &                    0.080 \\
Ours, no $\mathcal{R}_o$          &                    0.008 &                    0.086 &                    0.051 &                    0.161 &                    0.082 &                    0.076 \\
Ours, no pred. normals         & \cellcolor{orange} 0.006 &                    0.134 &                    0.064 &                    0.272 &                    0.087 & \cellcolor{tablered}    0.068 \\
Ours, concat. viewdir          & \cellcolor{orange} 0.006 &                    0.095 & \cellcolor{tablered}    0.040 &                   0.061 & \cellcolor{yellow} 0.079 &                    0.087 \\
Ours, fixed lobe               &                    0.009 &                    0.096 &                    0.043 &                   0.061 &                    0.080 &                    0.080 \\ \hline
Ours, no diffuse color         &                    0.009 &                    0.096 &                    0.043 &                    0.095 & \cellcolor{yellow} 0.079 &                    0.080 \\
Ours, no tint                  &                    0.008 &                    0.094 &                    0.043 & \cellcolor{orange} 0.060 & \cellcolor{yellow} 0.079 &                    0.078 \\
Ours, no roughness             &                    0.009 &                    0.099 &                    0.042 &                    0.086 &                    0.081 &                        0.074 \\
Ours, PE                       &                    0.007 &                    0.092 & \cellcolor{orange} 0.041 & \cellcolor{orange} 0.060 & \cellcolor{tablered}    0.077 &                       0.074 \\
Ours                           & \cellcolor{tablered}    0.004 &                    0.095 & \cellcolor{orange} 0.041 & \cellcolor{tablered}    0.059 & \cellcolor{orange} 0.078 &                    0.075 \\
    \end{tabular}
    }
    \vspace{-0.1in}
    \caption{
    Per-scene test set LPIPS on our Shiny Blender dataset.
    }
    \label{tab:shiny_lpips_allscenes}

    \resizebox{\linewidth}{!}{
    \begin{tabular}{@{}l|ccccccr}
     & \!teapot\! & \!toaster\! & \!car\! & \!ball\! & \!coffee\! & \!helmet\!  \\ \hline
PhySG~\cite{zhang2021physg}               & \cellcolor{tablered} 6.634  & \cellcolor{tablered} 9.749 & \cellcolor{tablered}  8.844 & \cellcolor{tablered}  0.700 &                22.514 & \cellcolor{tablered} 2.324 \\ \hline
Mip-NeRF~\cite{barron2021mipnerf}         &                      66.470 &                    42.787 &                    40.954 &                    104.765 &                    29.427 &                    77.904 \\
Mip-NeRF, 8 layers                        &                      68.238 &                    35.220 &                    23.670 &                    127.863 &                    25.465 &                    67.966 \\
Mip-NeRF, 8 layers, w/ normals            &                      67.999 &                    34.093 &                    23.548 &                    130.755 &                    26.527 &                    66.701 \\
Mip-NeRF, 8 layers,  with $\mathcal{R}_o$ &                      68.238 &                    35.837 &                    23.985 &                    127.683 &                    24.101 &                    64.372 \\ \hline
Ours, no reflection                       & \cellcolor{yellow}   19.263 & \cellcolor{orange} 19.325 & \cellcolor{yellow} 13.643 &                    15.142 & \cellcolor{orange} 9.730 & \cellcolor{yellow} 20.038 \\
Ours, no $\mathcal{R}_o$                  &                      43.116 &                    43.113 &                    37.134 &                    106.003 &                    29.301 &                    56.710 \\
Ours, no pred. normals                    &                      67.999 & \cellcolor{yellow} 24.886 &                    21.644 &                    48.061 &                    10.848 & \cellcolor{orange}    10.573 \\
Ours, concat. viewdir                     &                      20.359 &                    40.587 & \cellcolor{orange}    12.877 &                  1.577 & \cellcolor{tablered}    9.489 &                    42.615 \\
Ours, fixed lobe                          &                      35.791 &                    42.183 &                    18.410 & \cellcolor{orange}  1.536 &                    24.045 &                    36.785 \\ \hline
Ours, no diffuse color                    &                      38.347 &                    42.705 &                    16.802 &                    5.423 &                    15.363 &                    38.119 \\
Ours, no tint                             &                      29.537 &                    43.687 &                    16.854 &                    1.556 &                    11.233 &                    33.327 \\
Ours, no roughness                        &                      33.821 &                    44.666 &                    17.440 &                    3.557 &                    26.578 &                    29.723 \\
Ours, PE                                  &                      23.123 &                    41.415 &                    16.214 &                    1.969 & \cellcolor{yellow} 9.749 &                     29.409 \\
Ours                                      & \cellcolor{orange}   9.234 &                     42.870 &                  14.927 & \cellcolor{yellow}   1.548 &                    12.240 &                    29.484 \\
    \end{tabular}
    }
    \vspace{-0.1in}
    \caption{
    Per-scene test set MAEs on our Shiny Blender dataset.
    }
    \label{tab:shiny_maes_allscenes}
\end{table}

\end{document}